\documentclass[10pt]{article} %
\usepackage[accepted]{tmlr}

\usepackage{ifluatex}

\ifluatex
\usepackage{fontspec}
\else
\usepackage[utf8]{inputenc}
\usepackage[T1]{fontenc}
\fi
\usepackage{newunicodechar}

\usepackage{microtype}
\usepackage[super]{nth}
\usepackage{url}
\usepackage{nicefrac}

\usepackage{enumitem}

\usepackage{caption}
\usepackage[newfloat,frozencache]{minted}
\usepackage{booktabs}
\usepackage{multirow}
\usepackage{tablefootnote}
\usepackage{tabularx}
\usepackage[export]{adjustbox}

\usepackage{algorithm}
\usepackage[spaceRequire=false,italicComments=false]{algpseudocodex}

\algrenewcommand\algorithmicindent{1.2em}
\algrenewcommand\textproc{}%
\algrenewcommand\algorithmicrequire{\textbf{inputs:}}%
\algrenewcommand\algorithmicensure{\textbf{outputs:}}%
\NewCommandCopy{\Oalgorithmiccomment}{\algorithmiccomment}%
\renewcommand{\algorithmiccomment}[1]{\small{\Oalgorithmiccomment{#1}}}%

\usepackage{amsfonts}
\usepackage{amsmath}
\usepackage{amssymb}
\usepackage{amsthm}

\newtheorem{lemma}{Lemma}

\usepackage{mathtools}
\usepackage{thmtools}
\usepackage{thm-restate}

\newcommand{\bunderbrace}[2]{%
\begin{array}[t]{@{}c@{}}
\underbrace{#1} \\
\mathclap{#2}
\end{array}
}

\usepackage{systeme}

\usepackage{xparse}

\definecolor{easy_blue}{RGB}{0,98,163}
\definecolor{easy_red}{RGB}{150,60,0}
\definecolor{easy_blue_bg}{RGB}{234,241,248}
\definecolor{easy_red_bg}{RGB}{253,246,235}
\definecolor{blue_bg}{RGB}{0,61,86}

\usepackage{soul}

\usepackage{hyperref}%
\hypersetup{
hidelinks,
colorlinks=true,
linkcolor=easy_blue,
citecolor=easy_blue,
urlcolor=easy_red,
}

\usepackage[page]{appendix}

\usepackage[capitalise,noabbrev]{cleveref}%
\crefname{equation}{}{}
\Crefname{assumption}{Assumption}{Assumptions}
\crefmultiformat{equation}{(#2#1#3}{,~#2#1#3)}{,~#2#1#3}{,~#2#1#3)}

\allowdisplaybreaks{}

\NewDocumentCommand{\G}{e{_^}}{%
\mathop{}\!%
\nabla{}
\IfValueT{#1}{_{\!#1}}%
\IfValueT{#2}{^{#2}}%
}
\newcommand{\blank}{{\mspace{2mu}\cdot\mspace{2mu}}}

\newcommand{\argmin}{\operatorname*{arg\,min}}

\newcommand*{\tr}{^{\mkern-1.5mu\mathsf{T}}}%
\newcommand*{\hc}{^{\mathsf{H}}}%

\newcommand{\E}{\operatorname*{\mathbb{E}}}
\newcommand{\V}{\operatorname*{\mathbb{V}}}
\newcommand{\C}{\operatorname*{\mathbb{C}}}
\newcommand{\ind}{\operatorname*{\perp\!\!\!\perp}}

\newunicodechar{…}{\dots}
\newunicodechar{•}{\bullet}%
\newunicodechar{□}{\blank}%
\newunicodechar{○}{\circ}

\newunicodechar{×}{\times}
\newunicodechar{∂}{\partial}
\newunicodechar{∇}{\G}
\newunicodechar{∫}{\int}
\newunicodechar{∑}{\sum}
\newunicodechar{∏}{\prod}

\newunicodechar{¬}{\neg}
\newunicodechar{∅}{\varnothing}
\newunicodechar{∈}{\in}
\newunicodechar{∉}{\nin}
\newunicodechar{∃}{\exists}
\newunicodechar{∄}{\nexists}
\newunicodechar{∀}{\forall}
\newunicodechar{⊂}{\subset}
\newunicodechar{⊃}{\supset}
\newunicodechar{⊆}{\subseteq}
\newunicodechar{⊇}{\supseteq}
\newunicodechar{∪}{\cup}
\newunicodechar{⋓}{\bigcup}
\newunicodechar{∩}{\cap}
\newunicodechar{⋒}{\bigcap}

\newunicodechar{⊕}{\oplus}
\newunicodechar{⊖}{\ominus}
\newunicodechar{⊗}{\otimes}
\newunicodechar{⊘}{\oslash}
\newunicodechar{⊙}{\odot}

\newunicodechar{←}{\leftarrow}
\newunicodechar{→}{\rightarrow}
\newunicodechar{⟵}{\longleftarrow}
\newunicodechar{⟶}{\longrightarrow}
\newunicodechar{⇒}{\Rightarrow}
\newunicodechar{⇐}{\Leftarrow}
\newunicodechar{↦}{\mapsto}

\newunicodechar{≔}{\coloneqq}
\newunicodechar{≈}{\approx}
\newunicodechar{∝}{\propto}
\newunicodechar{≠}{\neq}

\newunicodechar{ℜ}{\Re}
\newunicodechar{ℑ}{\Im}
\newunicodechar{∞}{\infty}

\newunicodechar{≪}{\ll}
\newunicodechar{≫}{\gg}
\newunicodechar{≤}{\leq}
\newunicodechar{≥}{\geq}

\newunicodechar{∼}{\sim}
\newunicodechar{⫫}{\ind}
\newunicodechar{‖}{\Vert}

\newunicodechar{⋅}{\cdot}%
\newunicodechar{ᵀ}{\tr}
\newunicodechar{ᴴ}{\hc}

\newunicodechar{Δ}{\Delta}
\newunicodechar{Γ}{\Gamma}
\newunicodechar{Λ}{\Lambda}
\newunicodechar{Ω}{\Omega}
\newunicodechar{Φ}{\Phi}
\newunicodechar{Π}{\Pi}
\newunicodechar{Ψ}{\Psi}
\newunicodechar{Σ}{\Sigma}
\newunicodechar{Θ}{\Theta}
\newunicodechar{ϒ}{\Upsilon}
\newunicodechar{Ξ}{\Xi}
\newunicodechar{α}{\alpha}
\newunicodechar{β}{\beta}
\newunicodechar{χ}{\chi}
\newunicodechar{δ}{\delta}
\newunicodechar{ϵ}{\epsilon}
\newunicodechar{η}{\eta}
\newunicodechar{γ}{\gamma}
\newunicodechar{ι}{\iota}
\newunicodechar{κ}{\kappa}
\newunicodechar{λ}{\lambda}
\newunicodechar{μ}{\mu}
\newunicodechar{ν}{\nu}
\newunicodechar{ω}{\omega}
\newunicodechar{ϕ}{\phi}
\newunicodechar{π}{\pi}
\newunicodechar{ψ}{\psi}
\newunicodechar{ρ}{\rho}
\newunicodechar{σ}{\sigma}
\newunicodechar{τ}{\tau}
\newunicodechar{θ}{\theta}
\newunicodechar{υ}{\upsilon}
\newunicodechar{ε}{\varepsilon}
\newunicodechar{φ}{\varphi}
\newunicodechar{ϖ}{\varpi}
\newunicodechar{ϱ}{\varrho}
\newunicodechar{ς}{\varsigma}
\newunicodechar{ϑ}{\vartheta}
\newunicodechar{ξ}{\xi}
\newunicodechar{ζ}{\zeta}

\newunicodechar{𝛥}{\varDelta}
\newunicodechar{𝛤}{\varGamma}
\newunicodechar{𝛬}{\varLambda}
\newunicodechar{𝛺}{\varOmega}
\newunicodechar{𝛷}{\varPhi}
\newunicodechar{𝛱}{\varPi}
\newunicodechar{𝛹}{\varPsi}
\newunicodechar{𝛴}{\varSigma}
\newunicodechar{𝛩}{\varTheta}
\newunicodechar{𝛶}{\varUpsilon}
\newunicodechar{𝛯}{\varXi}

\newunicodechar{𝔸}{\mathbb{A}}
\newunicodechar{𝔹}{\mathbb{B}}
\newunicodechar{ℂ}{\C}
\newunicodechar{𝔻}{\mathbb{D}}
\newunicodechar{𝔼}{\E}
\newunicodechar{𝔽}{\mathbb{F}}
\newunicodechar{𝔾}{\mathbb{G}}
\newunicodechar{ℍ}{\mathbb{H}}
\newunicodechar{𝕀}{\mathbb{I}}
\newunicodechar{𝕁}{\mathbb{J}}
\newunicodechar{𝕂}{\mathbb{K}}
\newunicodechar{𝕃}{\mathbb{L}}
\newunicodechar{𝕄}{\mathbb{M}}
\newunicodechar{ℕ}{\mathbb{N}}
\newunicodechar{𝕆}{\mathbb{O}}
\newunicodechar{ℙ}{\mathbb{P}}
\newunicodechar{ℚ}{\mathbb{Q}}
\newunicodechar{ℝ}{\mathbb{R}}
\newunicodechar{𝕊}{\mathbb{S}}
\newunicodechar{𝕋}{\mathbb{T}}
\newunicodechar{𝕌}{\mathbb{U}}
\newunicodechar{𝕍}{\V}
\newunicodechar{𝕎}{\mathbb{W}}
\newunicodechar{𝕏}{\mathbb{X}}
\newunicodechar{𝕐}{\mathbb{Y}}
\newunicodechar{ℤ}{\mathbb{Z}}

\newunicodechar{𝒜}{\mathcal{A}}
\newunicodechar{ℬ}{\mathcal{B}}
\newunicodechar{𝒞}{\mathcal{C}}
\newunicodechar{𝒟}{\mathcal{D}}
\newunicodechar{ℰ}{\mathcal{E}}
\newunicodechar{ℱ}{\mathcal{F}}
\newunicodechar{𝒢}{\mathcal{G}}
\newunicodechar{ℋ}{\mathcal{H}}
\newunicodechar{ℐ}{\mathcal{I}}
\newunicodechar{𝒥}{\mathcal{J}}
\newunicodechar{𝒦}{\mathcal{K}}
\newunicodechar{ℒ}{\mathcal{L}}
\newunicodechar{ℳ}{\mathcal{M}}
\newunicodechar{𝒩}{\mathcal{N}}
\newunicodechar{𝒪}{\mathcal{O}}
\newunicodechar{𝒫}{\mathcal{P}}
\newunicodechar{𝒬}{\mathcal{Q}}
\newunicodechar{ℛ}{\mathcal{R}}
\newunicodechar{𝒮}{\mathcal{S}}
\newunicodechar{𝒯}{\mathcal{T}}
\newunicodechar{𝒰}{\mathcal{U}}
\newunicodechar{𝒱}{\mathcal{V}}
\newunicodechar{𝒲}{\mathcal{W}}
\newunicodechar{𝒳}{\mathcal{X}}
\newunicodechar{𝒴}{\mathcal{Y}}
\newunicodechar{𝒵}{\mathcal{Z}}

\newcommand{\cf}{\mathtt{f}}
\newcommand{\cb}{\mathtt{b}}
\newcommand{\cpf}{\mathtt{πf}}
\newcommand{\cpb}{\mathtt{πb}}
\newcommand{\ct}{\mathtt{t}}
\newcommand{\cl}{\mathtt{l}}
\newcommand{\TO}{\mathbin{‖}}
\newcommand{\fms}[2]{$\underset{\textcolor{gray}{#2}}{#1}$}
\newcommand{\feq}{\!=\!}
\newcommand{\ra}[1]{\renewcommand{\arraystretch}{#1}}

\newcommand{\bo}{\boldmath}

\title{BM$^2$: Coupled Schrödinger Bridge Matching}

\author{\name Stefano Peluchetti \email stepelu@sakana.ai \\
      Sakana AI}

\def\month{12}  %
\def\year{2024} %
\def\openreview{\url{https://openreview.net/forum?id=fqkq1MgONB}} %

\begin{document}

\maketitle

\begin{abstract}
A Schrödinger bridge establishes a dynamic transport map between two target distributions via a reference process, simultaneously solving an associated entropic optimal transport problem.
We consider the setting where samples from the target distributions are available, and the reference diffusion process admits tractable dynamics.
We thus introduce Coupled Bridge Matching (BM$^2$), a simple \emph{non-iterative} approach for learning Schrödinger bridges with neural networks.
A preliminary theoretical analysis of the convergence properties of BM$^2$ is carried out, supported by numerical experiments that demonstrate the effectiveness of our proposal.
\end{abstract}

\section{Introduction}\label{sec:intro}

The Schrödinger bridge problem seeks a process, the Schrödinger bridge, with prescribed initial and terminal distributions, such that the distribution of the Schrödinger bridge minimizes the Kullback-Leibler (KL) divergence to the distribution of a reference process.
Schrödinger bridges play a central role in measure transport theory \citep{marzouk2016sampling}.
Notably, it is known that the initial-terminal distribution of a Schrödinger bridge provides a solution to a corresponding entropic optimal transport problem \citep{peyre2020computational}.
Schrödinger bridges thus provide an effective framework for finding an alignment between samples from two target distributions.
Furthermore, diffusion-based generative models \citep{ho2020denoising,song2021scorebased} can be interpreted as solving trivial instances of the Schrödinger bridge problem \citep{peluchetti2023diffusion}.
Consequently, Schrödinger bridges offer a more general approach to contemporary generative applications.

We consider the setting where samples are readily available from both target distributions, and where the reference process is a diffusion process solution to a stochastic differential equation (SDE).
We thus introduce Coupled Bridge Matching (BM$^2$), a novel methodology aimed at computing the Schrödinger bridge given the reference SDE and samples from the two marginal distributions of interest.
BM$^2$ builds upon Bridge Matching (BM), introduced\footnote{\citet{peluchetti2021nondenoising} used the term ``Diffusion Bridge Mixture-Matching Transport'' (DBMT), but we follow \citet{shi2023diffusiona} in using the sleeker nomenclature ``Bridge Matching'' for this transport.} by \citet{peluchetti2021nondenoising}.
Our approach advances recent contributions by \citet{peluchetti2023diffusion,shi2023diffusiona} by removing the need to solve a sequence of optimization problems.
A neural network is employed to jointly learn a forward drift function and a backward drift function corresponding to the forward and backward dynamics of a Schrödinger bridge.
BM$^2$ achieves several key desiderata:
\begin{enumerate}[label= (\roman*)]
\item non-iterative: training is conducted through standard stochastic gradient descent within a single optimization loop;
\item exact: the idealized version of BM$^2$ yields the target Schrödinger bridge without approximations; the only sources of error involved in its practical implementation are the neural network approximation error and the discretization error due to sampling the learned SDE;\@
\item efficient gradient: the gradient of the loss function with respect to neural network parameters depends solely on few random variables sampled at the current optimization step;
\item simple loss: the loss function avoids derivative terms with respect to neural network inputs and does not impose hard constraints (such as conservative vector field requirements) on the neural network approximator.
\end{enumerate}
These features collectively enhance the efficiency and applicability of BM$^2$ in solving Schrödinger bridge problems.
Training is robust, as it does not depend on hyperparameters that are typically challenging to set without time-consuming pilot runs, such as the number of training steps per optimization iteration (i) or the level of approximation (ii).
Moreover, the memory requirements are modest due to (iii).
Finally, the implementation is straightforward (i, iv), as illustrated in \cref{alg:bm2_obj,alg:bm2_train} and in the annotated PyTorch code of \cref{code:bm2}.

\textbf{Content}: This paper is structured as follows.
In \cref{sec:reference_intro}, we formally introduce the Schrödinger bridge problem with associated reference process dynamics.
\cref{sec:bm} reviews Bridge Matching, while \cref{sec:bm2} introduces Coupled Bridge Matching, discussing its theoretical properties and implementation aspects.
Numerical experiments are presented in \cref{sec:numerical}, followed by a discussion of related works in \cref{sec:related_works}.
\cref{sec:conclusions} concludes the paper.
For clarity, a more general formulation of BM$^2$ is deferred to \cref{sec:additional_dynamics}, all proofs to \cref{sec:proofs}, an additional numerical experiment to \cref{sec:additional_results}, and code listings to \cref{sec:code}.

\textbf{Notation and Assumptions}: To enhance accessibility, we refrain from discussing the more technical aspects related to the Schrödinger bridge problem in its path measure formulation.
The excellent treaties of \citet{leonard2014survey,leonard2014properties} and \citet[Appendices D, H]{bortoli2021diffusion} already serve this goal.
We denote distributions with uppercase letters and their corresponding (Lebesgue) densities with lowercase letters.
All stochastic processes considered are $d$-dimensional, continuous, and defined on the unit time interval $[0,1]$.
For a stochastic process $X$ with distribution $P$ (denoted $X ∼ P$), we use subscripts to specify marginal distributions, joint distributions, and conditional distributions of $P$.
$P_t$: marginal distribution of $X_t$ at time $t$, with density $p_t$; $P_{0,1}$: initial-terminal joint distribution of $(X_0, X_1)$; $P_{{}|0}$: distribution of $X$ given its initial value $X_0$.
Superscripts indicate a distribution $P$'s dependency on another distribution $Z$, as in $P^Z$, or a sequence of distributions, as in $P^{(i)}$, $i ≥ 1$.
For a $d$-dimensional distribution $Q_0$, we define the stochastic process mixture distribution $Q_0 P_{{}|0}$ as: $(Q_0P_{{}|0})(X ∈ □) ≔ ∫P_{{}|0}(X ∈ □|x_0)Q_0(d x_0)$.
From a generative perspective, $X ∼ Q_0 P_{{}|0}$ is obtained by sampling $X_0 ∼ Q_0$ and then $X ∼ P_{{}|0}(□|X_0)$ conditionally on $X_0$.
The marginal-conditional decomposition of $P$ over its initial value is thus $P = P_0 P_{{}|0}$.
Similarly, for a $d{×}d$-dimensional joint distribution $Q_{0,1}$, we define $Q_{0,1} P_{{}|0,1}$ such that $X \sim Q_{0,1} P_{{}|0,1}$ is obtained by sampling $(X_0,X_1) ∼ Q_{0,1}$ and then $X ∼ P_{{}|0,1}(□|X_0,X_1)$ conditionally on $X_0$ and $X_1$.
Time is always indexed on a common forward timescale, on which all stochastic processes' distributions are defined.
The dynamics of a diffusion process $X ∼ P$ can be formulated in both forward and backward time directions, through corresponding forward and backward SDEs.
In backward SDEs, $t$ decreases from $1$ to $0$ ($dt$ is negative), which is denoted by $t∈[1,0]$.
All Brownian motions are independent.
Unless otherwise noted, each diffusion process is a Markov diffusion process which is a (weak) solution to an associated SDE.\@

\section{Problem Setting}\label{sec:reference_intro}

\subsection{Schrödinger Bridges and Entropic Optimal Transport}\label{sec:reference_sb}

For two target $d$-dimensional distributions $Ψ_0$ and $Ψ_1$, and a reference stochastic process distribution $R$, the \emph{dynamic} Schrödinger bridge (SB) problem seeks to find
\begin{equation}\label{eq:sb_problem}
S^{Ψ_0,Ψ_1,R} ≔ \argmin_{P ∈ 𝒫(Ψ_0,Ψ_1)}𝕂𝕃(P \TO R),
\end{equation}
where $𝕂𝕃(□ \TO □)$ is the KL divergence and $𝒫(Ψ_0,Ψ_1)$ is the class of distributions of stochastic processes having initial distribution $Ψ_0$ and terminal distribution $Ψ_1$.
We narrow down \cref{eq:sb_problem} to the case where $R$ is the distribution of a diffusion process.
In this case, under suitable conditions \citep{leonard2014survey}, \cref{eq:sb_problem} admits a unique solution which is also a diffusion process.
From this point forward, $Ψ_0$, $Ψ_1$ and $R$ are considered fixed.
For brevity, we will thus denote the Schrödinger bridge $S^{Ψ_0,Ψ_1,R}$ simply as $S$, and apply the same notation convention to any distribution dependent on these variables.

The forward and backward dynamics of $X ∼ S$ are given by:
\begin{align}
 & X_0 ∼ Ψ_0,\quad dX_t = μ_s(X_t,t)dt + σdW_t,\quad t∈[0,1],\tag{$S$}\label{eq:sb_fwd}                  \\
 & X_1 ∼ Ψ_1,\quad dX_t = -υ_s(X_t,t)dt + σdW_t,\quad t∈[1,0],\tag{$\overleftarrow{S}$}\label{eq:sb_bwd}
\end{align}
for the SB-optimal drift functions $μ_s, υ_s$.
These functions are related to the Schrödinger potentials \citep{leonard2014survey} and are not analytically available aside from very specific choices of $Ψ_0,Ψ_1$ and $R$.

We assume that $R_{0,1}$ admits density $r_{0,1}$.
Once $S$ is obtained, the solution to the \emph{static} Schrödinger bridge problem is given by $S_{0,1}$:
\begin{equation}\begin{aligned}\label{eq:sb_static_problem}
S_{0,1} & = \argmin_{C_{0,1} ∈ 𝒞(Ψ_0,Ψ_1)}𝕂𝕃(C_{0,1} \TO R_{0,1}),                          \\
        & = \argmin_{C_{0,1} ∈ 𝒞(Ψ_0,Ψ_1)}𝔼_{C_{0,1}}[-\log r_{1|0}(X_1|X_0)] - ℍ(C_{0,1}).
\end{aligned}\end{equation}
In \cref{eq:sb_static_problem}, $𝒞(Ψ_0,Ψ_1)$ denotes the class of $d{×}d$-dimensional joint distributions with marginal distributions $Ψ_0$ and $Ψ_1$, commonly referred to as the class of couplings of $Ψ_0$ and $Ψ_1$, and $ℍ(C_{0,1}) ≔ 𝔼_{C_{0,1}}[-\log c_{1,0}(X_1,X_0)]$ is the entropy of $C_{0,1}$.

The entropic optimal transport (EOT) solution for the cost function $k(x_0,x_1)$ and regularization level $ε$ is given by:
\begin{equation}\label{eq:eot_problem}
E_{0,1} ≔ \argmin_{C_{0,1} ∈ 𝒞(Ψ_0,Ψ_1)} 𝔼_{C_{0,1}}[κ(X_1,X_0)] - εℍ(C_{0,1}).
\end{equation}
Thus, for each choice of $R_{0,1}$ in \cref{eq:sb_static_problem}, $S_{0,1}$ solves a corresponding problem \cref{eq:eot_problem}.
As in the following, when $R$ is associated to \cref{eq:ref_sde}, $S_{0,1}$ solves the EOT problem \cref{eq:eot_problem} for the Euclidean cost $κ(x_1,x_0) = \nicefrac{1}{2}‖x_0 - x_1‖^2$ and regularization level $ε = σ^2$.

We refer to \citet{peyre2020computational,leonard2014survey,gushchin2023building} for related background material from complementary perspectives.

\subsection{Reference Dynamics}\label{sec:reference_dynamics}

We focus on the case where $R$ is the distribution of a scaled Brownian motion:
\begin{equation}\label{eq:ref_sde}
X_0 ∼ Ψ_0,\quad dX_t = σ dW_t,\quad t∈[0,1],\tag{$R$}
\end{equation}
with $σ > 0$.
Our approach \emph{is not limited} to the choice of SDE \cref{eq:ref_sde}, BM$^2$ readily extends to the broader class of reference SDEs examined in \citet{peluchetti2023diffusion}.
The main requirement for the applicability of BM$^2$ is the analytical availability of \cref{eq:ref_sde_td,eq:ref_sde_tdd} for the chosen reference SDE.\@
We address the case, commonly employed in generative applications, of $dX_t = σ\sqrt{β_t}dW_t$ for a schedule $β_t$ explicitly in \cref{sec:additional_dynamics}, and refer the reader to \citet{peluchetti2021nondenoising,peluchetti2023diffusion} for the general setting.
As our developments are orthogonal to the specific choice of reference process, we focus on the simplest case for explanatory reasons.

We collect here various results concerning \cref{eq:ref_sde} that will be utilized in the following:
\begin{align}
 & R_{t|0}(□|x_0) = 𝒩(x_0,σ^2t),\label{eq:ref_sde_td}                                                                       \\
 & R_{t|0,1}(□|x_0,x_1) = 𝒩(x_0(1 - t) + x_1t, σ^2t(1 - t)),\label{eq:ref_sde_tdd}                                          \\
 & μ_{01}(x_t,t,x_1) ≔ σ^2 ∇_{x_t} \log r_{1|t}(x_1|x_t) = \frac{x_1 - x_t}{1 - t},\label{eq:μ01}                           \\
 & υ_{01}(x_t,t,x_0) ≔ σ^2∇_{x_t}\log r_{t|0}(x_t|x_0) = \frac{x_0 - x_t}{t},\label{eq:υ01}                                 \\
 & γ_{01}(x_t,t,x_0,x_1) ≔ σ^2∇_{x_t}\log r_{t|0,1}(x_t|x_0,x_1) = \frac{x_0(1 - t) + x_1 t - x_t}{t(1 - t)}.\label{eq:γ01}
\end{align}

Conditioning $X ∼ R$ on the endpoints $X_0=x_0$, $X_1=x_1$ results in the diffusion bridge distribution $R_{{}|0,1}$, with associated forward and backward SDEs:
\begin{align}
 & X_0=x_0, \quad dX_t = μ_\mathrm{01}(X_t,t,x_1)dt + σdW_t,\quad t∈[0,1],\tag{$R_{{}|0,1}$}\label{eq:sde_bridge}               \\
 & X_1=x_1, \quad dX_t = -υ_{01}(X_t,t,x_0)dt + σdW_t,\quad t∈[1,0].\tag{$\overleftarrow{R_{{}|0,1}}$}\label{eq:sde_bridge_rev}
\end{align}

\section{Bridge Matching (BM)}\label{sec:bm}

We succinctly review Bridge Matching, and refer to \citet{peluchetti2021nondenoising,peluchetti2023diffusion,shi2023diffusiona} for more details.
BM takes as input a joint distribution $Q_{0,1}$ with marginal distributions $Q_0,Q_1$ and a SDE, \cref{eq:ref_sde}.
Firstly, a stochastic process $Π^{Q_{0,1}}$ is constructed as a mixture of diffusion bridges \cref{eq:sde_bridge}, such that the endpoints $(X_0,X_1)$ of $X ∼ Π^{Q_{0,1}}$ are distributed according to $Q_{0,1}$.
This process, which is a mixture of diffusion processes, is not itself a diffusion process in general \citep{jamison1974reciprocal}.
However, we can obtain a marginal-matching diffusion process with distribution $M^{Q_{0,1}}$ for which $M^{Q_{0,1}}_t = Π^{Q_{0,1}}_t, 0 ≤ t ≤ 1$.
Consequently, $X ∼ M^{Q_{0,1}}$ is a diffusion process for which $X_0 ∼ Q_0$ and $X_1 ∼ Q_1$, i.e.\ it defines a dynamic transport from $Q_0$ to $Q_1$.

Concretely, let $Π^{Q_{0,1}} ≔ Q_{0,1}R_{{}|0,1}$.
The BM transport based on $Q_{0,1}$ with distribution $M^{Q_{0,1}}$ is realized by
\begin{align}
 & X_0 ∼ Q_0,\quad dX_t = \bunderbrace{μ_m^{Q_{0,1}}(X_t,t)}{𝔼_{Π^{Q_{0,1}}}[μ_{01}(X_t,t,X_1)|X_t]}dt + σdW_t,\quad t∈[0,1],\tag{$M$}\label{eq:bm_f_sde}                  \\
 & X_1 ∼ Q_1,\quad dX_t = -\bunderbrace{υ_m^{Q_{0,1}}(X_t,t)}{𝔼_{Π^{Q_{0,1}}}[υ_{01}(X_t,t,X_0)|X_t]}dt + σdW_t,\quad t∈[1,0],\tag{$\overleftarrow{M}$}\label{eq:bm_b_sde}
\end{align}
and satisfies $M^{Q_{0,1}}_t = Π^{Q_{0,1}}_t$, $0 ≤ t ≤ 1$.

As conditional expectations are mean squared error minimizers, suitable training objectives for the drift functions $μ_m^{Q_{0,1}}$ and $υ_m^{Q_{0,1}}$ are derived from
\begin{align}
 & μ_m^{Q_{0,1}} = \argmin_{μ} 𝔼_{Π^{Q_{0,1}}}\Big[\frac{1}{2}∫_0^1‖μ_{01}(X_t,t,X_1) - μ(X_t,t)‖^2dt\Big],\label{eq:bm_f_obj} \\
 & υ_m^{Q_{0,1}} = \argmin_{υ} 𝔼_{Π^{Q_{0,1}}}\Big[\frac{1}{2}∫_0^1‖υ_{01}(X_t,t,X_0) - υ(X_t,t)‖^2dt\Big],\label{eq:bm_b_obj}
\end{align}
by replacing each integral with an expectation over uniform time $t ∼ 𝒰(0, 1)$, and then approximating both expectations with Monte Carlo estimators.
While we will rely exclusively on \cref{eq:bm_f_obj,eq:bm_b_obj} in the experiments of \cref{sec:numerical}, $μ_m^{Q_{0,1}}$ and $υ_m^{Q_{0,1}}$ can be inferred from paths $X ∼ Π^{Q_{0,1}}$ also by performing maximum likelihood estimation or by employing a drift matching estimator \citep{liu2022let,peluchetti2023diffusion}.

We conclude this section by reviewing prior BM results relevant for BM$^2$.
Define:
\begin{equation*}\begin{aligned}
 & 𝒫 ≔ \{d\text{-dimensional, continuous, stochastic processes on }[0,1]\},                                   \\
 & ℛ ≔ \{P ∈ 𝒫 \mathrel{|} P = P_{0,1}R_{{}|0,1} = Π^{P_{0,1}}\text{ for some }P_{0,1}\},                     \\
 & ℳ ≔ \{P ∈ 𝒫 \mathrel{|} P\text{ is a (Markov) diffusion process}\},                                        \\
 & 𝒮 ≔ \{P ∈ 𝒫 \mathrel{|} P\text{ is a Schrödinger bridge for some target marginal distributions}\} = ℛ ∩ ℳ,
\end{aligned}\end{equation*}
where the equivalence is established by \citet{jamison1975markov} under appropriate assumptions.
We additionally define the following restrictions: $𝒫(Ψ_0,⋅) ≔ \{P ∈ 𝒫 \mathrel{|} P_0 = Ψ_0\}$, $𝒫(⋅,Ψ_1) ≔ \{P ∈ 𝒫 \mathrel{|} P_1 = Ψ_1\}$, $𝒫(Ψ_0,Ψ_1) ≔ \{P ∈ 𝒫 \mathrel{|} P_0 = Ψ_0 \text{ and } P_1 = Ψ_1\}$.
Restrictions to $ℛ,ℳ,𝒮$ and $𝒞$ employ the same notation.

For $Q ∈ 𝒫$, it is instructive to view BM as a map between distributions defined by the composition of two projections: $Q \overset{ℛp}{→} Π^{Q_{0,1}}\overset{ℳp}{→} M^{Q_{0,1}}$.
Here, the \emph{reciprocal projection} $ℛp: 𝒫 → ℛ$ projects $Q$ onto the reciprocal class $ℛ$, while the \emph{Markovian projection} $ℳp: ℛ → ℳ$ projects $Π^{Q_{0,1}}$ onto the class of diffusion processes, see \citet{shi2023diffusiona}.
It follows that if $P ∈ ℛ$, then $P = ℛp(P)$, and if $P ∈ ℳ$, then $P = ℳp(P)$.
Consequently, if $P ∈ 𝒮$, $P = (ℳp ○ ℛp)(P)$ for the BM map $(ℳp ○ ℛp)$, and conversely if $P = (ℳp ○ ℛp)(P)$ then $P ∈ 𝒮$.

\subsection{Iterated Bridge Matching (I-BM) and Diffusion Iterative Proportional Fitting (DIPF)}\label{sec:bm_iterated}

In the dynamic setting, \citet{peluchetti2023diffusion,shi2023diffusiona} demonstrate that, under suitable conditions, iterative application of the BM procedure to an initial coupling $C_{0,1} ∈ 𝒞(Ψ_0,Ψ_1)$ results in convergence toward $S$.
Specifically, defining $I^{(0)} ≔ M^{C_{0,1}}$ and $I^{(i)} ≔ M^{I^{(i-1)}_{0,1}}$ for $i ≥ 1$, it holds that $𝕂𝕃(I^{(i)} \TO S) → 0$ as $i → ∞$.
In practical applications, the independent initial coupling given by the product distribution $C_{0,1} = Ψ_0{⊗}Ψ_1$ is frequently employed.

In the static setting, the classical procedure employed in solving problems \cref{eq:sb_static_problem,eq:eot_problem} is known by several names: the Sinkhorn algorithm \citep{peyre2020computational}, the Iterated Proportional Fitting (IPF) procedure \citep{ruschendorf1995convergence}, or Fortet iterations \citep{fortet1940resolution}.
The iterates are given by $D_{0,1}^{(0)} ≔ Ψ_0R_{1|0}$, $D_{0,1}^{(1)} ≔ Ψ_1K^{(0)}_{0|1}$, $D_{0,1}^{(2)} ≔ Ψ_0D^{(1)}_{1|0}$, and so on.
At each iteration, one of the target marginal distributions is replaced while the remaining conditional distribution is kept fixed.
Alternatively, one can start from $Ψ_1R_{0|1}$ following the same logic.
Under suitable conditions \citep{ruschendorf1995convergence}, KL convergence $𝕂𝕃(S_{0,1} \ \TO D_{0,1}^{(i)}) → 0$ is established.
The key insight of \citet{bortoli2021diffusion,vargas2021solving} is that it is possible to extend the IPF iterations to the dynamic setting.
In this case, the IPF iterations are implemented by learning the time reversal of a diffusion process at each iteration.
We refer to the resulting training algorithm, as proposed by \citet{bortoli2021diffusion}, as Diffusion Iterative Proportional Fitting (DIPF).
\citet{bortoli2021diffusion} establishes the convergence properties of the DIPF iterates, see their Propositions 4, 5 and Section 3.5.

\section{Coupled BM (BM$^2$)}\label{sec:bm2}

As a starting point for the derivation of BM$^2$, consider the system of equations
\begin{equation}\label{eq:bm2_system}\left\{\begin{aligned}
H^{K'_{0,1}} = Ψ_0M^{K'_{0,1}}_{{}|0} \\
K^{H'_{0,1}} = Ψ_1M^{H'_{0,1}}_{{}|1}
\end{aligned}\right.,\end{equation}
whose variables are diffusion distributions $H^{K'_{0,1}},K^{H'_{0,1}}$ and $H',K'$.
That is, $H^{K'_{0,1}}$ is obtained as the BM transport based on $K'_{0,1}$ conditioned to have initial distribution $Ψ_0$, while $K^{H'_{0,1}}$ is obtained as the BM transport based on $H'_{0,1}$ conditioned to have terminal distribution $Ψ_1$.
Equivalently, \cref{eq:bm2_system} is expressed as
\begin{align}
 & X_0 ∼ Ψ_0,\quad dX_t = μ_m^{K'_{0,1}}(X_t,t)dt + σdW_t,\quad t∈[0,1],\tag{$H^{K'_{0,1}}$}\label{eq:bm2_h}                  \\
 & X_1 ∼ Ψ_1,\quad dX_t = -υ_m^{H'_{0,1}}(X_t,t)dt + σdW_t,\quad t∈[1,0].\tag{$\overleftarrow{K}^{H'_{0,1}}$}\label{eq:bm2_k}
\end{align}
All of $μ_m,υ_m,M$ are defined in \cref{sec:bm}.
System \cref{eq:bm2_system} defines an update step $(H',K') \overset{\cref{eq:bm2_system}}{→} (H^{K'_{0,1}},K^{H'_{0,1}})$.
We are interested in the fixed points of such updates, i.e.\ $(H',K')$ such that $(H',K') \overset{\cref{eq:bm2_system}}{→} (H',K')$.
It holds that $H' = K' = S$ is a fixed point to \cref{eq:bm2_system}.
As $S ∈ 𝒮(Ψ_0,Ψ_1)$, $S = Π^{S_{0,1}} = M^{S_{0,1}}$, see the review at the end of \cref{sec:bm}.
Consequently, $Ψ_0M_{{}|0}^{S_{0,1}} = Ψ_0S_{{}|0} = S$ and $Ψ_1M_{{}|1}^{S_{0,1}} = Ψ_1S_{{}|1} = S$.
In this case, the SB-optimal drifts $μ_s$ and $υ_s$ of \cref{eq:sb_fwd,eq:sb_bwd} respectively replace $μ_m^{K'_{0,1}}$ and $υ_m^{H'_{0,1}}$ in \cref{eq:bm2_h,eq:bm2_k}.
Under the additional assumption that $H'=K'$, or equivalently that \cref{eq:bm2_h,eq:bm2_k} are the time reversal of each other, this fixed point is unique.
Let $G = H' = K'$, we have $G = Ψ_0 M^{G_{0,1}}_{{}|0} = G_0 M^{G_{0,1}}_{{}|0} = M^{G_{0,1}}_0 M^{G_{0,1}}_{{}|0} = M^{G_{0,1}}$ and $G_0 = Ψ_0, G_1 = Ψ_1$, thus $G = 𝒮(Ψ_0,Ψ_1) = S$.
We have shown the following:

\begin{lemma}[Fixed points of \cref{eq:bm2_system}]\label{thm:system_fixed}
Under suitable conditions \citep{leonard2014properties}, the updates $(H',K') \overset{\cref{eq:bm2_system}}{→} (H^{K'_{0,1}},K^{H'_{0,1}})$, parametrized by diffusion process distributions, admit $H' = K' = S$ as fixed point.
If $H' = K'$, this fixed point is unique.
\end{lemma}

When $μ_m^{K'_{0,1}} = μ_s$ and $υ_m^{H'_{0,1}} = υ_s$, \cref{eq:bm2_system} has reached an equilibrium.
The updates $(H',K') \overset{\cref{eq:bm2_system}}{→} (H^{K'_{0,1}},K^{H'_{0,1}})$ are realized through the computation of the drifts $μ_m^{K'_{0,1}}$ and $υ_m^{H'_{0,1}}$, i.e.\ by minimizing the losses \cref{eq:bm_f_obj,eq:bm_b_obj}, where $Q_{0,1}$ is respectively equal to $K'_{0,1}$ and $H'_{0,1}$.
Our proposal, BM$^2$, follows from replacing the complete minimization of \cref{eq:bm_f_obj,eq:bm_b_obj} with partial and stochastic minimization of \cref{eq:bm_f_obj,eq:bm_b_obj} through stochastic gradient descent.
More precisely, consider the forward and backward SDEs with distributions $F(θ)$ and $B(θ)$:
\begin{align}
 & X_0 ∼ Ψ_0,\quad dX_t = μ_f(X_t,t,θ)dt + σdW_t,\quad t∈[0,1],\tag{$F(θ)$}\label{eq:bm2_f}                  \\
 & X_1 ∼ Ψ_1,\quad dX_t = -υ_b(X_t,t,θ)dt + σdW_t,\quad t∈[1,0].\tag{$\overleftarrow{B}(θ)$}\label{eq:bm2_b}
\end{align}
$μ_f(X_t,t,θ)$ and $υ_b(X_t,t,θ)$ are drift functions to be learned, which are implemented through a neural network with parameters $θ$.
Let $θ'$ represent the values of $θ$ at a given step during training, and define the losses
\begin{equation}\label{eq:bm2_losses}\begin{aligned}
 & 𝕃_f(θ;θ') ≔ 𝔼_{Π^{B_{0,1}(θ')}}\Big[\frac{1}{2}∫_0^1‖μ_{01}(X_t,t,X_1) - μ_f(X_t,t,θ)‖^2dt\Big], \\
 & 𝕃_b(θ;θ') ≔ 𝔼_{Π^{F_{0,1}(θ')}}\Big[\frac{1}{2}∫_0^1‖υ_{01}(X_t,t,X_1) - υ_b(X_t,t,θ)‖^2dt\Big], \\
 & 𝕃(θ;θ') ≔ 𝕃_f(θ;θ') + 𝕃_b(θ;θ').
\end{aligned}\end{equation}
At each optimization step, BM$^2$ attempts to minimize $𝕃(θ;θ')$ in $θ$ via a step of stochastic gradient descent, starting from $θ = θ'$ and keeping $θ'$ fixed, resulting in $θ''$.
The subsequent optimization step employs $θ' ← θ''$.
The complete training objective is presented in \cref{alg:bm2_obj}, where $\mathrm{sg()}$ refers to the stop-gradient operator --- $𝕃(θ;θ')$ is minimized in the first arguments only --- and $\mathrm{discretize()}$ represents a generic SDE discretization scheme.
For completeness, we outline the standard SGD training loop in \cref{alg:bm2_train}, where $\mathrm{sgdstep()}$ refers to an update step via a generic gradient descent optimizer.

It should be noted that merely performing coupled drift matching of $F$ and $B$, wherein $F$ learns the drift consistent with paths from $B$ and vice versa, does not yield the Schrödinger bridge as a fixed point \citep{bortoli2021diffusion}.
The introduction of the mixing process $Π$ is crucial in ensuring this property.
Moreover, $𝕃(θ;θ')$ must be minimized only with respect to its first argument: the application of the stop-gradient operator $\mathrm{sg()}$ is not an efficiency consideration but a necessary component.
\noindent
\begin{algorithm}
\caption{BM$^2$ --- training loss computation}\label{alg:bm2_obj}
\begin{algorithmic}[1]
\Ensure{$\cl(θ)$: sampled loss value}
\Require{$θ$: current parameters}
\Function{loss}{$θ$}
\State{$\cf_0 ∼ Ψ_0$}\Comment{Marginal sampling}
\State{$\cf_{Δt},…,\cf_1|\cf_0 ∼ \mathrm{sg}(\mathrm{discretize}(\cf_0,Δt,μ_f(□,□,θ)))$}\Comment{Discretization of \cref{eq:bm2_f}}
\State{$\cb_1 ∼ Ψ_1$}\Comment{Marginal sampling}
\State{$\cb_{1 - Δt},…,\cb_0|\cb_1 ∼ \mathrm{sg}(\mathrm{discretize}(\cb_1,Δt,υ_b(□,□,θ)))$}\Comment{Discretization of \cref{eq:bm2_b}}
\State{$\ct ∼ 𝒰(0,1)$}\Comment{Time sampling}
\State{$\cpf_\ct ∼ R_{\ct|0,1}(□|\cf_0,\cf_1)$}\Comment{Bridge sampling \cref{eq:ref_sde_tdd}}
\State{$\cpb_\ct ∼ R_{\ct|0,1}(□|\cb_0,\cb_1)$}\Comment{Bridge sampling \cref{eq:ref_sde_tdd}}
\State{$\cl_\cf(θ) ← \nicefrac{1}{2}‖μ_{01}(\cpb_\ct,\ct,\cb_1) - μ_f(\cpb_\ct,\ct,θ)‖^2$}\Comment{BM based on $B_{0,1}$ \cref{eq:bm_f_obj,eq:μ01}}
\State{$\cl_\cb(θ) ← \nicefrac{1}{2}‖υ_{01}(\cpf_\ct,\ct,\cf_0) - υ_b(\cpf_\ct,\ct,θ)‖^2$}\Comment{BM based on $F_{0,1}$ \cref{eq:bm_b_obj,eq:υ01}}
\State{$\cl(θ) ← \cl_\cf(θ) + \cl_\cb(θ)$}
\State{\Return{$\cl(θ)$}}
\EndFunction{}
\end{algorithmic}
\end{algorithm}

\begin{algorithm}
\caption{BM$^2$ --- training loop}\label{alg:bm2_train}
\begin{algorithmic}[1]
\Ensure{$θ^*$: trained parameters}
\Require{$θ^○$: initial parameters}
\Function{train}{$θ^○$}
\State{$θ ← θ^○$}
\While{not converged}
\State{$\cl(θ) ← \mathrm{loss}(θ)$}\Comment{Sample loss with \cref{alg:bm2_obj}}
\State{$θ ← \text{sgdstep}(θ, ∇_{θ} \cl(θ))$}\Comment{Perform SGD step}
\EndWhile{}
\State{\Return{$θ$}}
\EndFunction{}
\end{algorithmic}
\end{algorithm}

\subsection{Implementation Aspects}\label{sec:bm2_implementation}

The following aspects are not presented in \cref{code:bm2}, but impacts the performance of BM$^2$.

\textbf{Path Caching}: as in \citet{bortoli2021diffusion,shi2023diffusiona}, to enhance efficiency, we cache the initial and terminal endpoints of the paths sampled in lines 3 and 5 of \cref{alg:bm2_obj}, and periodically refresh the cache during training.
Notably, it is unnecessary to cache entire paths; only the endpoints are required for bridge sampling, which is advantageous from a memory perspective.
Bridge sampling offers the additional benefit of increased sample diversity: for cached (fixed) endpoints, the samples corresponding to lines 7 and 8 differ at each step.

\textbf{Model}: we utilize a single neural network to parametrize both $μ_f(x,t,θ)$ and $υ_b(x,t,θ)$.
As the training process is not iterative, it is unnecessary to introduce multiple neural networks (or parameters), one for each iteration.\@

\textbf{Sampling EMA}: as in \citet{ho2020denoising,song2021scorebased}, to improve the stability of training we apply the Exponential Moving Averaging (EMA) to the parameters employed in path sampling in lines 3 and 5 of \cref{alg:bm2_obj}.

\textbf{Loss Singularities}: the losses of lines 9 and 10 of \cref{alg:bm2_obj} diverge for $\ct → 1$ and $\ct → 0$ respectively.
Singularities of these kind are common to scalable losses for generative diffusion models.
In our numerical experiments we simply restrict sampling of $t$ to $𝒰(ϵ,1-ϵ)$ for a small $ϵ > 0$.
More sophisticated alternatives involve either employing the dynamics of \cref{sec:additional_dynamics} for an appropriate scheduling $β_t$, or learning terminal-value predictors in place of drift terms, recovering the latter through \cref{eq:μ01,eq:υ01}.

\textbf{Two-Stage Training}: a significant challenge in early training is the simulation-inference mismatch.
To achieve reliable results, the drift functions $μ_f(x_t,t,θ)$ and $υ_b(x_t,t,θ)$ must be accurately learned in regions where the corresponding SDEs \cref{eq:bm2_f} and \cref{eq:bm2_b} will be simulated.
However, the processes $F(θ)$ and $B(θ)$ typically differ at initialization.
Because the drift of \cref{eq:bm2_f} is inferred from samples of \cref{eq:bm2_b} (and vice versa), the approximation quality can be poor; see \citet[Sections 2.3 and 6.2]{peluchetti2023diffusion} for a detailed analysis of this issue within the DIPF framework.
To address this challenge, the BM transport based on the independent coupling $Ψ_0{⊗}Ψ_1$ can be learned in both directions in a first stage.
By construction, the BM transport circumvents the simulation-inference mismatch.
Moreover, at convergence, the processes $F(θ)$ and $B(θ)$ are the same.
Subsequently, in the second stage, the BM$^2$ transport can be learned employing the first-stage solution as initialization.

\textbf{Forward-Backward Consistency}: the mutual time reversal relationship between \cref{eq:bm2_f} and \cref{eq:bm2_b}, i.e., the equivalence of $F(θ)$ and $B(θ)$, can be encouraged leveraging the diffusion time-reversal result of \citet{anderson1982reversetime}, yielding the additional consistency loss term:
\begin{equation}\label{eq:consistency}
ℒ_{f,b}(θ;θ') ≔ 𝔼_{\frac{1}{2}(Π^{F_{0,1}}(θ') + Π^{B_{0,1}}(θ'))}\Big[\frac{1}{2}\int_0^1‖μ_f(X_t,t,θ) + υ_b(X_t,t,θ) - γ_{01}(X_t,t,X_0,X_1)‖^2dt\Big],
\end{equation}
where $θ'$ denotes an independent copy of $θ$ (implemented via the stop-gradient operator), and $γ_{01}(x_t,t,x_0,x_1)$ is defined in \cref{eq:γ01}.
Indeed, for a mixture process $Π^{C_{0,1}} = C_{0,1}R_{|0,1}$, the following relationship holds:
\begin{equation*}
σ^2 ∇\log π_t^{C_{0,1}}(x_t) = 𝔼_{Π^{C_{0,1}}}[γ_{01}(X_t,t,X_0,X_1)|X_t=x_t].
\end{equation*}
Although \cref{eq:consistency} shares similarities with the consistency loss proposed by \citet[page 33]{shi2023diffusiona}, \cref{eq:consistency} utilizes the dynamic mixture $\frac{1}{2}(Π^{F(θ')_{0,1}} + Π^{B(θ')_{0,1}}) = Π^{\frac{1}{2}(F(θ')_{0,1} + B(θ')_{0,1})}$.
When a single neural network is employed to parametrize both $μ_f(x_t,t,θ)$ and $υ_b(x_t,t,θ)$, estimating \cref{eq:consistency} in addition to \cref{eq:bm2_losses} results in a negligible computational overhead per SGD step.

\subsection{Convergence Properties}\label{sec:bm2_theory}

At each training step, BM$^2$ performs a partial and stochastic minimization of the loss $𝕃(θ;θ')$ from \cref{eq:bm2_losses} with respect to $θ$, where $𝕃(θ;θ')$ is defined by an expectation over a distribution dependent on $θ'$, yielding $θ''$.
Subsequently, $θ'$ is updated to match $θ''$, and the process advances to the next training step.
The alternation between expectation and maximization steps bears resemblance to the classical Expectation-Maximization (EM) algorithm \citep{dempster1977maximum}.

\subsubsection{Complete Minimization}

We start by establishing in \cref{thm:bm2full} that the version of BM$^2$ where $𝕃(θ;θ')$ is fully minimized at each training step recovers the I-BM and DIPF iterations for two specific initialization choices of \cref{eq:bm2_f,eq:bm2_b}.
The prior convergence results of \citet{bortoli2021diffusion,shi2023diffusiona,peluchetti2023diffusion} (see the review of \cref{sec:bm_iterated}) toward $S$ thus apply.

To facilitate the presentation of the convergence results in this section, we introduce, with a slight abuse of notation, the following functional versions of the losses \cref{eq:bm2_losses}:
\begin{equation}\label{eq:bm2_losses_fun}\begin{aligned}
 & 𝕃_f(μ_f;υ'_b) ≔ 𝔼_{Π^{B'_{0,1}}}\Big[\frac{1}{2}∫_0^1‖μ_{01}(X_t,t,X_1) - μ_f(X_t,t)‖^2dt\Big], \\
 & 𝕃_b(υ_b;μ'_f) ≔ 𝔼_{Π^{F'_{0,1}}}\Big[\frac{1}{2}∫_0^1‖υ_{01}(X_t,t,X_1) - υ_b(X_t,t)‖^2dt\Big], \\
 & 𝕃(μ_f,υ_b;μ'_f,υ'_b) ≔ 𝕃_f(μ_f;υ'_b) + 𝕃_b(υ_b;μ'_f).
\end{aligned}\end{equation}
In \cref{eq:bm2_losses_fun} we identify $μ_f,υ_b$ with $F,B$, and $μ_f',υ'_b$ with $F',B'$ (the remaining quantities defining $F,B,F',B'$ are fixed).
We will use $𝕃_f(μ_f;υ'_b)$, $𝕃_b(υ_b;μ'_f)$ and $𝕃_f(θ;θ')$, $𝕃_b(θ;θ')$ interchangeably.
We are now ready to state our first convergence result.

\begin{restatable}[Complete BM$^2$ Iterations]{theorem}{bmtwofull}\label{thm:bm2full}
Consider the SDEs \cref{eq:bm2_f,eq:bm2_b}, with initial drifts $μ_f^{(0)}, υ_b^{(0)}$ and corresponding distributions $F^{(0)}, B^{(0)}$.
For each $i ≥ 1$, let $(μ_f^{(i)}, υ_b^{(i)}) = \argmin_{(μ,υ)}𝕃(μ,υ;μ_f^{(i-1)},υ_b^{(i-1)})$, resulting in the distribution iterates $F^{(i)},B^{(i)}$.
We distinguish two cases:
\begin{enumerate}[label= (\roman*)]
\item $μ_f^{(0)} = υ_b^{(0)} = 0$: both the iterates $F^{(0)},B^{(1)},F^{(2)},…$ and the iterates $B^{(0)},F^{(1)},B^{(2)},…$ are equivalent to the DIPF iterates, started respectively from the forward and from the backward time direction;
\item $μ_f^{(0)} = μ_m^{C_{0,1}}, υ_b^{(0)} = υ_m^{C_{0,1}}$ for some $C ∈ 𝒞(Ψ_0,Ψ_1)$: $F^{(i)} = B^{(i)} = I^{(i)}$ for each $i ≥ 0$ where $I^{(i)}$ are the I-BM iterates.
\end{enumerate}
\end{restatable}

\subsubsection{Partial Minimization}

In the EM algorithm it suffices to perform partial maximization steps.
A partial result for the setting where $𝕃(θ;θ')$ is partially minimized with respect to $θ$ at each step is stated in \cref{thm:bm2partial}, which is based on \cref{thm:system_fixed} and \cref{thm:lossint}.

\begin{restatable}[Loss Interpretation]{lemma}{lossint}\label{thm:lossint}
It holds that
\begin{equation}\begin{aligned}\label{eq:loss_to_kl}
 & 𝕂𝕃(B_0 \TO Ψ_0) + 𝕃_f(μ_f;υ_b) = 𝕂𝕃(Π^{B_{0,1}} \TO F) + C_1(B) = 𝕂𝕃(M^{B_{0,1}} \TO F) + C_2(B), \\
 & 𝕂𝕃(F_1 \TO Ψ_1) + 𝕃_b(υ_b;μ_f) = 𝕂𝕃(Π^{F_{0,1}} \TO B) + D_1(F) = 𝕂𝕃(M^{F_{0,1}} \TO B) + D_2(F),
\end{aligned}\end{equation}
for $C_1(B),C_2(B)$ independent of $F$, $D_1(F),D_2(F)$ independent of $B$, with $0 ≤ C_1(B) ≤ C_2(B)$ and $0 ≤ D_1(F) ≤ D_2(F)$.
\end{restatable}

The losses $𝕃_f(μ_f;υ_b)$ and $𝕃_b(υ_b;μ_f)$ are easily amenable to optimization in their first arguments, as seen in \cref{alg:bm2_obj}.
\cref{thm:lossint} relates these losses to more interpretable KL divergences between distributions.
By \cref{eq:loss_to_kl}, a decrease of $𝕃_f(μ_f;υ_b)$ due to a change in $μ_f$ corresponds to equivalent decreases of $𝕂𝕃(M^{B_{0,1}} \TO F)$ for a fixed $υ_b$, or $B$.
Thus, partial minimization of $𝕃_f(μ_f;υ_b)$ brings $F$ closer to $M^{B_{0,1}}$, the BM transport based on $B_{0,1}$, and the result of a complete minimization step, by means of reverse KL minimization.
Symmetric considerations apply to $𝕃_b(υ_b;μ_f)$ as function of its first argument.
Putting this result and \cref{thm:system_fixed} together yields \cref{thm:bm2partial}.

\begin{restatable}[Partial BM$^2$ Iterations]{theorem}{bmtwopartial}\label{thm:bm2partial}
At each optimization step, decreases of $𝕃_f(θ;θ')$ and $𝕃_b(θ;θ')$ in $θ$ correspond to equivalent decreases of $𝕂𝕃(M^{B_{0,1}(θ')} \TO F(θ))$ and $ 𝕂𝕃(M^{F_{0,1}(θ')} \TO B(θ))$.
If the losses $𝕃_f(θ;θ')$ and $𝕃_b(θ;θ')$ cannot be decreased in $θ$, i.e., at optimality, and if $F(θ) = B(θ)$, then $F(θ) = B(θ) = S$.
\end{restatable}

\subsubsection{Infinitesimal Minimization}

We conclude our theoretical investigation by relating our proposal to the work of \citet{karimi2023sinkhorn}, which introduces a continuous variant of the IPF procedure.
In IPF, the two target marginal distributions are replaced sequentially, one at a time.
Each step corresponds to solving a static Schrödinger \emph{half-bridge} problem \citep{leonard2014properties}, where in \cref{eq:sb_static_problem}, $𝒞(Ψ_0,Ψ_1)$ is replaced by either $𝒞(Ψ_0,⋅)$ or $𝒞(⋅,Ψ_1)$.
The approach proposed by \citet{karimi2023sinkhorn} retains either the even or odd steps of the IPF scheme while substituting the alternate steps with partial minimizations of the corresponding half-bridge problems.
In the limit of infinitesimally small improvements, this yields a dynamical system for the evolution of the iterates over continuous algorithmic time.

We demonstrate that a similar result can be obtained for a modified version of BM$^2$, where forward KL divergences are minimized instead of reverse KL divergences.
The resulting dynamical system is a symmetrized version of the one obtained by \citet{karimi2023sinkhorn}.
Let $F'$, $B'$ represent the current state in the optimization process.
We consider a partial minimization of $𝕂𝕃(F \TO M^{B'_{0,1}})$, instead of $𝕂𝕃(M^{B'_{0,1}} \TO F)$, in $F$ and a partial minimization of $𝕂𝕃(B \TO M^{F'_{0,1}})$, instead of $𝕂𝕃(M^{F'_{0,1}} \TO B)$, in $B$.
As in \citet{karimi2023sinkhorn}, partial minimization is formulated as
\begin{equation}\begin{aligned}\label{eq:bm2_partial}
 & F^{(λ)} ≔ \argmin_{F ∈ ℳ(Ψ_0,⋅)} λ𝕂𝕃(F \TO M^{B'_{0,1}}) + (1 - λ)𝕂𝕃(F \TO F'), \\
 & B^{(λ)} ≔ \argmin_{B ∈ ℳ(⋅,Ψ_1)} λ𝕂𝕃(B \TO M^{F'_{0,1}}) + (1 - λ)𝕂𝕃(B \TO B'),
\end{aligned}\end{equation}
where $λ ∈ [0, 1]$ controls the extent of the minimization.
We begin by establishing two stability results: the updates $(F',B') \overset{\cref{eq:bm2_partial}}{→}(F^{(λ)},B^{(λ)})$ preserve both $ℛ$ and $𝒮$.

\begin{restatable}[$ℛ$-stability of $F^{(λ)}, B^{(λ)}$]{lemma}{stabler}\label{thm:bm2_partial_r}
If $F',B' ∈ ℛ$, then $F^{(λ)}, B^{(λ)} ∈ ℛ$ for each $λ ∈ [0,1]$.
\end{restatable}

\begin{restatable}[$𝒮$-stability of $F^{(λ)}, B^{(λ)}$]{lemma}{stables}\label{thm:bm2_partial_s}
If $F',B' ∈ 𝒮$, then $F^{(λ)}, B^{(λ)} ∈ 𝒮$ for each $λ ∈ [0, 1]$.
\end{restatable}

Provided that the initial values $F',B' ∈ 𝒮$, \cref{thm:bm2_partial_s} establishes that the iterates defined by the updates $(F',B') \overset{\cref{eq:bm2_partial}}{→}(F^{(λ)},B^{(λ)})$ always remain in $𝒮$.
It is straightforward to ensure that $F',B' ∈ 𝒮$ at initialization by setting the corresponding drifts to zero: $μ_f',υ_b'=0$, which we will assume henceforth.
As $M^{B'_{0,1}} = B'$ and $M^{F'_{0,1}} = F'$, \cref{eq:bm2_partial} can be reformulated in simpler terms:
\begin{equation}\begin{aligned}\label{eq:bm2_partial_simple}
 & F^{(λ)} ≔ \argmin_{F ∈ ℳ(Ψ_0,⋅)} λ𝕂𝕃(F \TO B') + (1 - λ)𝕂𝕃(F \TO F'), \\
 & B^{(λ)} ≔ \argmin_{B ∈ ℳ(⋅,Ψ_1)} λ𝕂𝕃(B \TO F') + (1 - λ)𝕂𝕃(B \TO B').
\end{aligned}\end{equation}

By \cref{thm:bm2_partial_r}, it suffices to solve \cref{eq:bm2_partial_simple} in the static setting,
\begin{equation}\begin{aligned}\label{eq:bm2_partial_static}
 & F^{(λ)}_{0,1} ≔ \argmin_{F_{0,1} ∈ 𝒞(Ψ_0,⋅)} λ𝕂𝕃(F_{0,1} \TO B'_{0,1}) + (1 - λ)𝕂𝕃(F_{0,1} \TO F'_{0,1}), \\
 & B^{(λ)}_{0,1} ≔ \argmin_{B_{0,1} ∈ 𝒞(⋅,Ψ_1)} λ𝕂𝕃(B_{0,1} \TO F'_{0,1}) + (1 - λ)𝕂𝕃(B_{0,1} \TO B'_{0,1}).
\end{aligned}\end{equation}
The dynamic solutions are then recovered by $F^{(λ)}_{{}|0,1} = B^{(λ)}_{{}|0,1} = R_{{}|0,1}$.

We assume that $F'_{0,1}, B'_{0,1}, Ψ_0, Ψ_1$ admit densities.
By calculus of variations, the solution to \cref{eq:bm2_partial_static} is given by $f^{(λ)}_{0,1}(x_0,x_1) = ψ_0(x_0)f^{(λ)}_{1|0}(x_1|x_0)$, and $b^{(λ)}_{0,1}(x_0,x_1) = b^{(λ)}_{0|1}(x_0|x_1)ψ_1(x_1)$, where $f^{(λ)}_{1|0}(x_1|x_0) ∝ b'_{1|0}(x_1|x_0)^λ f'_{1|0}(x_1|x_0)^{1-λ}$ and $b^{(λ)}_{0|1}(x_0|x_1) ∝ f'_{0|1}(x_0|x_1)^λ b'_{0|1}(x_0|x_1)^{1-λ}$.
The IPF iterations are recovered when $λ=1$.
Instead, taking the limit $λ → 0$ and applying Bayes theorem, we obtain the evolution of $\log f^{(l)}_{1|0}(x_1|x_0)$ and $\log b^{(l)}_{0|1}(x_0|x_1)$ as a function of algorithmic time $l ∈ [0, ∞)$ through the dynamical system
\begin{equation}\begin{aligned}\label{eq:bm2_partial_static_dyn_2}
 & \frac{d\log f^{(l)}_{1|0}(x_1|x_0)}{dl} = -\log\frac{f^{(l)}_{1|0}(x_1|x_0)}{b^{(l)}_{0|1}(x_0|x_1)ψ_1(x_1)} + \overline{𝕂𝕃}(f^{(l)}_{1|0}(x_1|x_0) \TO b^{(l)}_{0|1}(x_0|x_1)ψ_1(x_1)),\quad l∈[0,∞), \\
 & \frac{d\log b^{(l)}_{0|1}(x_0|x_1)}{dl} = -\log\frac{b^{(l)}_{0|1}(x_0|x_1)}{f^{(l)}_{1|0}(x_1|x_0)ψ_0(x_0)} + \overline{𝕂𝕃}(b^{(l)}_{0|1}(x_0|x_1) \TO f^{(l)}_{1|0}(x_1|x_0)ψ_0(x_0)),\quad l∈[0,∞).
\end{aligned}\end{equation}
In \cref{eq:bm2_partial_static_dyn_2}, $\overline{𝕂𝕃}(□ \TO □)$ denotes the generalized KL divergence between unnormalized densities, as is the case here for the second arguments, and the initial conditions $f^{(0)}_{1|0}(x_1|x_0)$ and $b^{(0)}_{0|1}(x_0|x_1)$ are determined by \cref{eq:bm2_f,eq:bm2_b} with null drift terms.
\cref{eq:bm2_partial_static_dyn_2} can be contrasted with \citet[Equation (13)]{karimi2023sinkhorn}.
In \cref{sec:bm2_inf_gauss} we report a simple numerical application of \cref{eq:bm2_partial_static_dyn_2} to the Gaussian setting, which recovers $S$.

\section{Numerical Experiments}\label{sec:numerical}

To evaluate the performance of BM$^2$ on EOT problems, we utilize the benchmark developed by \citet{gushchin2023building}.
For the reference process \cref{eq:ref_sde}, this benchmark provides pairs of target distributions $Ψ_0,Ψ_1$ with analytical EOT solution $S_{0,1}$ and analytical SB-optimal drift function $μ_s$.
We focus on the mixtures benchmark, which consist of a centered Gaussian distribution as $S_0 = Ψ_0$ and a mixture of 5 Gaussian distributions for $S_{1|0}$.
$S_1 = Ψ_1$ is not a mixture of Gaussian distributions, but has 5 distinct modes.
The benchmark is constructed for dimensions $d ∈ \{2, 16, 64, 128\}$ and entropic regularization parameters $ε ∈ \{0.1, 1, 10\}$.

For each fully trained method, characterized by a stochastic process distribution $P$ and forward drift function $μ_p$, we assess performance using two evaluation metrics:
\begin{itemize}
\item $𝕂𝕃(S \TO P)$ where, by Girsanov theorem \citep{oksendal2013stochastic},
      \begin{equation}\label{eq:metric_kl}
      𝕂𝕃(S \TO P) = 𝔼_{S}\Big[\frac{1}{2σ^2}∫_0^1 ‖μ_s(X_t,t) - μ_p(X_t, t)‖^2dt\Big];
      \end{equation}
\item $\text{c}\text{B}𝕎_2^2\text{-UVP}(S_{0,1},P_{0,1})$, where
      \begin{equation}\label{eq:metric_cbw}
      \text{c}\text{B}𝕎_2^2\text{-UVP}(S_{0,1},P_{0,1}) ≔ \frac{100}{\frac{1}{2}𝕍_S[X_1]} ∫\text{B}𝕎_2^2(S_{1|0}(X_1|X_0),P_{1|0}(X_1|X_0))S_0(dX_0),
      \end{equation}
      $\text{B}𝕎_2^2(□,□)$ is the squared Bures-Wasserstein distance, i.e.\ the squared Wasserstein-2 distance between (assumed) multivariate Gaussian distributions \citep{dowson1982frechet}, and $𝕍_S[X_1]$ is the variance of $X_1 ∼ S_1$.
\end{itemize}

We focus on the divergence $𝕂𝕃(S \TO P)$, rather than $𝕂𝕃(P \TO S)$, as a low $𝕂𝕃(S \TO P)$ more accurately indicates that $P$ approximates $S$ effectively across the entire support of $S$.
The data-processing inequality implies that $𝕂𝕃(S_{0,1} \TO P_{0,1}) ≤ 𝕂𝕃(S \TO P)$.
The $\text{c}\text{B}𝕎_2^2\text{-UVP}(□,□)$ metric, introduced by \citet{gushchin2023building}, is a normalized and conditional extension of the standard $\text{B}𝕎_2^2(□,□)$ distance.
Results for evaluation metrics \cref{eq:metric_kl} and \cref{eq:metric_cbw} are summarized in \cref{tab:kl_results} and \cref{tab:cbw_results}, respectively.

\begin{table}
\ra{1.2}
\centering
\small
\begin{tabular}{@{}l*{12}r@{}}\toprule
         & \multicolumn{4}{c}{$ε \feq 0.1$} & \multicolumn{4}{c}{$ε \feq 1$} & \multicolumn{4}{c}{$ε \feq 10$}                                                                                                                                                                                                         \\
\cmidrule(lr){2-5} \cmidrule(lr){6-9} \cmidrule(l){10-13}
Method   & $d \feq 2$                       & $d \feq 16$                    & $d \feq 64$                     & $d \feq 128$        & $d \feq 2$          & $d \feq 16$         & $d \feq 64$         & $d \feq 128$        & $d \feq 2$          & $d \feq 16$         & $d \feq 64$          & $d \feq 128$         \\
\midrule
BM$^2$   & \bo\fms{0.01}{0.01}              & \bo\fms{0.20}{0.02}            & \bo\fms{1.03}{0.07}             & \bo\fms{3.06}{0.16} & \bo\fms{0.01}{0.00} & \bo\fms{0.11}{0.00} & \bo\fms{1.43}{0.03} & \fms{8.29}{0.36}    & \bo\fms{0.11}{0.01} & \bo\fms{2.25}{0.04} & \bo\fms{13.13}{0.13} & \bo\fms{40.46}{0.49} \\
BM$^2_σ$ & \fms{0.43}{0.09}                 & \fms{3.76}{0.46}               & \fms{39.55}{1.96}               & \fms{127.2}{1.4}    & \fms{0.04}{0.01}    & \fms{0.43}{0.03}    & \fms{5.36}{0.35}    & \fms{18.66}{0.73}   & \fms{0.15}{0.00}    & \fms{2.64}{0.05}    & \fms{13.78}{0.24}    & \fms{43.42}{1.43}    \\
\midrule[0.25pt]
I-BM     & \fms{0.03}{0.01}                 & \fms{0.20}{0.02}               & \fms{1.24}{0.04}                & \fms{5.70}{0.42}    & \fms{0.01}{0.00}    & \fms{0.16}{0.01}    & \fms{1.94}{0.04}    & \bo\fms{7.79}{0.07} & \fms{0.16}{0.00}    & \fms{4.09}{0.03}    & \fms{17.17}{0.21}    & \fms{49.17}{0.55}    \\
DIPF     & \fms{0.59}{0.14}                 & \fms{2.39}{0.05}               & \fms{7.93}{1.23}                & \fms{34.77}{0.82}   & \fms{0.23}{0.06}    & \fms{1.21}{0.18}    & \fms{13.13}{0.79}   & \fms{36.51}{1.05}   & \fms{0.81}{0.06}    & \fms{28.25}{2.12}   & \fms{113.8}{7.2}     & \fms{345.8}{8.1}     \\
\bottomrule
\end{tabular}
\caption{Monte Carlo estimate of $𝕂𝕃(S \TO P)$ as function of $ε$ and $d$, standard deviation in gray.}
\label{tab:kl_results}
\end{table}

In our benchmarking, we compare BM$^2$ against the I-BM and DIPF methods (\cref{sec:bm_iterated}).
Each experiment is repeated five times, including both model training and metric evaluation, to obtain uncertainty quantification.
We use $1,000$ Monte Carlo samples to estimate \cref{eq:metric_kl,eq:metric_cbw}.
For simplicity, we employ the Euler–Maruyama scheme \citep{kloeden1992numerical} with $200$ discretization steps ($Δt = 0.005$) in all path sampling procedures.
Each method undergoes $50,000$ SGD training steps with a batch size of $1,000$, settings similar to those used by \citet{gushchin2023building}, enabling qualitative comparison of our results with theirs.
We use the AdamW optimizer with a learning rate of $10^{-4}$ and hyperparameters: $β=(0.9, 0.999), ϵ=10^{-8}, wd=0.01$, where $wd$ denotes weight decay.
Time is sampled as $t ∼ 𝒰(ϵ, 1 - ϵ)$ for $ϵ = 0.0025$.

For BM$^2$, we employ a single feedforward neural network with $3$ layers of width $768$ and ReLU activation, resulting in approximately $1$ million parameters.
We initialize the neural network parameters such that the resulting initial forward and backward drift functions evaluate to zero everywhere, a choice that has proven effective in our experiments.
As mentioned in \cref{sec:bm2_implementation}, we implement path caching and an exponential moving average for parameters used in path sampling.
The cache contains $5,000$ initial-terminal values from both \cref{eq:bm2_f} and \cref{eq:bm2_b}, refreshed every $200$ training steps.
We omit reporting the results for both the two-stage training procedure and the consistency loss described in \cref{sec:bm2_implementation}, as neither approach demonstrated consistent performance improvements across the considered benchmark.

For I-BM and DIPF, each outer loop iteration comprises $5,000$ SGD steps, totaling $10$ outer loop (algorithmic) iterations.
Following best practices \citep{bortoli2021diffusion,shi2023diffusiona}, we alternate time directions over iterations for both algorithms.
Each method employs two separate neural networks for forward and backward time directions, maintaining a total parameter count close to $1$ million, matching BM$^2$'s model size.
As with BM$^2$, we implement path caching (for DIPF, entire discretized paths are cached) and EMA for sampling, with an EMA decay rate of $0.99$.

We also consider BM$^2_σ$, a variant of BM$^2$ that learns Schrödinger bridges for $Ψ_0,Ψ_1$ across multiple $σ$ values.
This amortized version leverages BM$^2$'s non-iterative nature.
At each optimization step, $σ$ is sampled from $𝒰(0.1, 4)$ and utilized in discretizing SDEs \cref{eq:bm2_f,eq:bm2_b} (lines 3 and 5 of \cref{alg:bm2_obj}) and in bridge sampling (lines 7 and 8 of \cref{alg:bm2_obj}).
The neural network implementing drift functions $μ_f(x,t,θ)$ and $υ_b(x,t,θ)$ is modified to accept $σ$ as an additional input, resulting in conditional drift functions $μ_f(x,t,θ,σ)$ and $υ_b(x,t,θ,σ)$.
Path caching is adjusted to store $σ$ values corresponding to cached paths.

In \cref{tab:cbw_results}, we additionally include three baselines.
EOT:\@ sampling from the EOT solution, accounting for the bias due to Monte Carlo estimation.
SB(discr): sampling from the SB solution via the SB-optimal drift $μ_s$, additionally accounting for Euler–Maruyama scheme discretization error.
$Ψ_0{⊗}Ψ_1$: sampling from the independent coupling.
Results for additional discretization intervals are reported in \cref{sec:cbw_discretizations}, and we illustrate in \cref{fig:training} the evolution of \cref{eq:metric_cbw} during training for a representative benchmark setting.

\begin{figure}[H]
\centering
\includegraphics{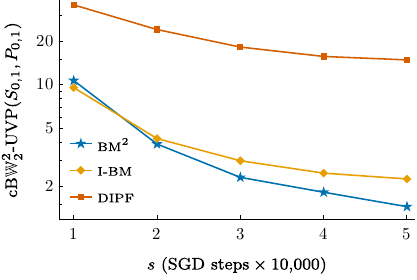}
\caption{Evolution of the metric \cref{eq:metric_cbw} for $d=64$, $ε=1$ over SGD steps for BM$^2$, I-BM and DIPF.}\label{fig:training}
\end{figure}

We now discuss the results presented in \cref{tab:kl_results,tab:cbw_results}.
BM$^2$ demonstrates superior overall performance across dimensions and entropic regularization settings in both metrics.
I-BM also shows good performance, particularly in comparison to the DIPF procedure, which aligns with the findings of \citet{shi2023diffusiona}.

As expected, the performance of all methods deteriorates as the number of dimensions increases.
This is because the metric\cref{eq:metric_kl} scales linearly with the number of dimensions, assuming a constant error rate in estimating each component of the true drift $μ_s$.
Similar considerations apply to the metric \cref{eq:metric_cbw}.

While BM$^2_σ$ exhibits a performance gap compared to BM$^2$, it yields reasonable results in low-dimensional settings ($d=2,16$).
This gap may be due to increased pressure on model capacity or the need to normalize loss levels across $σ$ values.
All methods perform poorly in the high regularization setting ($ε=10$), especially in high dimensions ($d=64,128$), which we include for completeness.
It should be noted that, in such cases, sampling from the independent coupling (a trivial solution) is preferable to sampling from the SB-optimal SDE for the chosen discretization interval.

\begin{table}
\ra{1.2}
\centering
\small
\begin{tabular}{@{}l*{12}r@{}}\toprule
            & \multicolumn{4}{c}{$ε \feq 0.1$} & \multicolumn{4}{c}{$ε \feq 1$} & \multicolumn{4}{c}{$ε \feq 10$}                                                                                                                                                                                                        \\
\cmidrule(lr){2-5} \cmidrule(lr){6-9} \cmidrule(l){10-13}
Method      & $d \feq 2$                       & $d \feq 16$                    & $d \feq 64$                     & $d \feq 128$        & $d \feq 2$          & $d \feq 16$         & $d \feq 64$         & $d \feq 128$        & $d \feq 2$          & $d \feq 16$         & $d \feq 64$          & $d \feq 128$        \\
\midrule
EOT         & \fms{0.02}{0.00}                 & \fms{0.05}{0.00}               & \fms{0.34}{0.00}                & \fms{0.91}{0.00}    & \fms{0.09}{0.00}    & \fms{0.17}{0.00}    & \fms{0.43}{0.00}    & \fms{1.14}{0.00}    & \fms{0.12}{0.00}    & \fms{0.18}{0.00}    & \fms{0.23}{0.00}     & \fms{0.38}{0.00}    \\
SB(discr.)  & \fms{0.04}{0.00}                 & \fms{0.07}{0.00}               & \fms{0.35}{0.00}                & \fms{0.92}{0.00}    & \fms{0.10}{0.00}    & \fms{0.17}{0.00}    & \fms{0.45}{0.00}    & \fms{1.18}{0.00}    & \fms{0.12}{0.00}    & \fms{1.15}{0.00}    & \fms{5.38}{0.01}     & \fms{10.48}{0.01}   \\
$Ψ_0{⊗}Ψ_1$ & \fms{195.8}{6.9}                 & \fms{186.3}{2.4}               & \fms{162.6}{0.8}                & \fms{145.1}{2.1}    & \fms{136.1}{4.7}    & \fms{127.6}{1.4}    & \fms{113.0}{1.7}    & \fms{93.61}{1.57}   & \fms{8.07}{0.33}    & \fms{4.88}{0.14}    & \fms{4.22}{0.09}     & \fms{4.45}{0.07}    \\
\midrule[0.25pt]
BM$^2$      & \bo\fms{0.73}{0.40}              & \fms{4.64}{0.58}               & \bo\fms{6.84}{0.59}             & \bo\fms{8.28}{0.62} & \bo\fms{0.14}{0.03} & \bo\fms{0.41}{0.04} & \bo\fms{1.72}{0.09} & \fms{8.30}{1.17}    & \bo\fms{0.14}{0.01} & \bo\fms{2.30}{0.04} & \bo\fms{41.14}{1.99} & \fms{264.4}{7.0}    \\
BM$^2_σ$    & \fms{8.38}{3.40}                 & \fms{16.06}{2.81}              & \fms{44.15}{0.84}               & \fms{83.84}{0.92}   & \fms{0.20}{0.07}    & \fms{2.61}{0.41}    & \fms{25.89}{1.65}   & \fms{64.76}{2.66}   & \fms{0.14}{0.01}    & \fms{2.57}{0.03}    & \fms{58.76}{0.76}    & \fms{323.0}{8.8}    \\
\midrule[0.25pt]
I-BM        & \fms{1.07}{0.50}                 & \bo\fms{4.25}{0.66}            & \fms{7.19}{0.28}                & \fms{16.63}{2.07}   & \fms{0.20}{0.09}    & \fms{0.53}{0.04}    & \fms{2.20}{0.35}    & \bo\fms{7.79}{0.79} & \fms{0.14}{0.02}    & \fms{5.21}{0.11}    & \fms{135.8}{1.3}     & \fms{578.7}{9.8}    \\
DIPF        & \fms{7.82}{2.51}                 & \fms{15.30}{1.00}              & \fms{20.12}{1.53}               & \fms{29.36}{1.02}   & \fms{1.66}{0.24}    & \fms{5.98}{0.65}    & \fms{13.11}{2.49}   & \fms{28.86}{3.52}   & \fms{0.69}{0.08}    & \fms{6.85}{0.21}    & \fms{72.63}{0.70}    & \bo\fms{226.1}{1.1} \\
\bottomrule
\end{tabular}
\caption{Monte Carlo estimate of $\text{cB}𝕎^2_2\text{-UVP}(S_{0,1},P_{0,1})$ as function of $ε$ and $d$, standard deviation in gray.}
\label{tab:cbw_results}
\end{table}

\section{Related Works}\label{sec:related_works}

Relevant works that, like BM$^2$, address the \emph{dynamic} Schrödinger bridge problem \cref{eq:sb_problem} include:

\textbf{Schrödinger Bridge Flow}: in a concurrent work, \citet{debortoli2024schrodinger} propose a non-iterative methodology called $α$-DSBM (Diffusion Schrödinger Bridge Matching).
$α$-DSBM is optimally implemented through a forward-backward SDE approach, in which case $α$-DSBM's formulation and training objective align with our proposal \cref{eq:bm2_f,eq:bm2_b,eq:bm2_losses}, see \citet[Equations (10) and (11)]{debortoli2024schrodinger}.
The formulation of $α$-DSBM begins by establishing a probability flow in the space of path probability measures, whose discretization yields $α$-IMF (Iterative Markovian Fitting).
Under mild conditions, two theoretical results are established: (i) the $α$-IMF iterates converge to the SB solution, and (ii) the non-parametric updates of the functional loss \cref{eq:bm2_losses_fun}, obtained through functional gradient descent with respect to the drift functions, recover the $α$-IMF iterates.
The practical implementation, $α$-DSBM, thus replaces non-parametric drift functions with parametric neural network approximators, and employs standard stochastic gradient descent on the parametric loss \cref{eq:bm2_losses}.
This theoretical framework provides strong convergence guarantees for $α$-DSBM and paves the way to further theoretical developments.
Furthermore, \citet{debortoli2024schrodinger} demonstrate the method's effectiveness through extensive numerical experiments on high-dimensional computer vision problems, complementing our synthetic benchmark results.

\textbf{I-BM and DIPF}: closely related to BM$^2$ are the iterative, sample-based DIPF \citep{bortoli2021diffusion,vargas2021solving} and I-BM \citep{shi2023diffusiona,peluchetti2023diffusion} procedures, which do not satisfy desiderata (i).
Built on similar bridge matching principles, BM$^2$ can be viewed as a modification of I-BM that employs a single optimization loop, resulting in a simpler algorithm that we have empirically shown to be competitive.

\textbf{Forward-Backward SB SDE}: \citet{chen2022likelihood} propose two training algorithms addressing the dynamic SB problem.
Both approaches employ loss functions that require divergence computations (violating desiderata (iv)) and the use of two distinct neural networks.
The first method is iterative, resembling DIPF (violating desiderata (i)), while the second method involves differentiating through entire discretized paths, resulting in high memory consumption (violating desiderata (iii)).

The subsequent works concentrate on solving the \emph{static} Schrödinger bridge \cref{eq:sb_static_problem}, or EOT \cref{eq:eot_problem}, problem.
Once this is achieved, solutions to the dynamic problem are trivially obtained through the standard decomposition $S = S_{0,1}R_{{}|0,1}$.
Although these works differ in nature and objectives, we include them here due to their shared characteristic with BM$^2$: the non-iterative nature of the algorithm.

\textbf{Light SB}: in two notable works, \citet{korotin2023light} and \citet{gushchin2024lighta} propose non-iterative, sample-based EOT solvers for the Euclidean cost function, i.e., for the specific choice of reference dynamics \cref{eq:ref_sde}.
\citet{korotin2023light} introduces an approximation to (an adjusted version of) the Schrödinger potential for $Ψ_1$ via a mixture of Gaussian distributions, resulting in a mixture of Gaussian distributions approximation to $S_{1|0}$.
\citet{gushchin2024lighta} builds upon this approximation and introduces an additional sample-based training objective that takes as input any coupling $C_{0,1} ∈ 𝒞(Ψ_0,Ψ_1)$, whereas \citet{korotin2023light} requires the independent coupling $Ψ_0{⊗}Ψ_0$.
While also non-iterative, the proposals of \citet{korotin2023light,gushchin2024lighta} differ from BM$^2$ in two key aspects: (a) they learn a solution in the static setting instead of the dynamic one, and (b) they employ mixture of Gaussian distributions approximations, rather than neural network approximators for the drift functions.
Consequently, these methods may face challenges in scaling to modern generative ML applications.
Light SB, in both variants, demonstrates strong performance in the benchmark presented in \cref{sec:numerical}: the results of \citep[Table 1]{gushchin2024lighta} and \citep[Table 2]{korotin2023light} are directly comparable with the results of \cref{tab:cbw_results}.
However, it is worth noting that this benchmark is particularly well-suited for Light SB, as acknowledged by its authors, since each benchmark target $S_{0,1}$ is constructed such that $S_{1|0}$ is itself a mixture of 5 Gaussian distributions.

\section{Conclusions}\label{sec:conclusions}

In this work we introduced Coupled Bridge Matching (BM$^2$), a novel approach for learning Schrödinger bridges from samples.
BM$^2$ builds on the principles of Bridge Matching while addressing key limitations of existing iterative methods.
Our approach offers several advantages, including a simple single-loop optimization procedure, exactness in the idealized setting, modest memory requirements, and a straightforward loss function.
The numerical experiments demonstrate that BM$^2$ is competitive with and often outperforms existing iterative diffusion-based methods like I-BM and DIPF across various dimensions and entropic regularization settings.

On the theoretical front, there is substantial room for improvement.
Firstly, while bearing some resemblance to the standard convergence result for the EM algorithm, \cref{thm:bm2partial} lacks a quantity analogous to the likelihood being maximized in the EM algorithm.
It remains unclear whether decreases in $𝕂𝕃(M^{B_{0,1}(θ')} \TO F(θ))$ and $ 𝕂𝕃(M^{F_{0,1}(θ')} \TO B(θ))$ can be linked to decreases in $𝕂𝕃(F(θ) \TO S)$ and $ 𝕂𝕃(B(θ) \TO S)$.
Secondly, the requirement that $F(θ) = B(θ)$, equivalently that \cref{eq:bm2_f} and \cref{eq:bm2_b} are time-reversals of each other, appears unnecessary.
Notably, all numerical simulations conducted do not explicitly enforce this condition, which emerges naturally during the training process.
Thirdly, it would be valuable to study problem \cref{eq:bm2_partial} where reverse KL divergences are partially minimized, aligning more closely with the BM$^2$ algorithm.
In this scenario, \cref{thm:bm2_partial_s} no longer holds, and it may be necessary to impose a corresponding additional constraint to maintain tractable analytical computations.
The attractors of \cref{eq:bm2_partial_static_dyn_2}, and of a corresponding dynamical system arising from reverse KL minimization, can be investigated to assess further convergence properties of BM$^2$.

On the empirical front, the applications of BM$^2$ in contemporary generative machine learning tasks remain unexplored.
Given the promising results from previous studies employing Bridge Matching, such as those by \citet{liu2023sb} and \citet{somnath23aligned}, it is anticipated that BM$^2$ could be effectively applied to various domains, including image generation, audio synthesis, and molecular design.
Future work could investigate the scalability and performance of BM$^2$ in these domains.

\clearpage
\bibliography{main}

\begin{thebibliography}{30}
\providecommand{\natexlab}[1]{#1}
\providecommand{\url}[1]{\texttt{#1}}
\expandafter\ifx\csname urlstyle\endcsname\relax
  \providecommand{\doi}[1]{doi: #1}\else
  \providecommand{\doi}{doi: \begingroup \urlstyle{rm}\Url}\fi

\bibitem[Anderson(1982)]{anderson1982reversetime}
Brian~D.O. Anderson.
\newblock Reverse-time diffusion equation models.
\newblock \emph{Stochastic Processes and their Applications}, 12\penalty0 (3):\penalty0 313--326, 1982.
\newblock ISSN 03044149.

\bibitem[Bortoli et~al.(2021)Bortoli, Thornton, Heng, and Doucet]{bortoli2021diffusion}
Valentin~De Bortoli, James Thornton, Jeremy Heng, and Arnaud Doucet.
\newblock Diffusion {{Schrödinger Bridge}} with {{Applications}} to {{Score-Based Generative Modeling}}.
\newblock In \emph{Thirty-{{Fifth Conference}} on {{Neural Information Processing Systems}}}, 2021.

\bibitem[Chen et~al.(2022)Chen, Liu, and Theodorou]{chen2022likelihood}
Tianrong Chen, Guan-Horng Liu, and Evangelos Theodorou.
\newblock Likelihood {{Training}} of {{Schrödinger Bridge}} using {{Forward-Backward SDEs Theory}}.
\newblock In \emph{International {{Conference}} on {{Learning Representations}}}, 2022.

\bibitem[De~Bortoli et~al.(2024)De~Bortoli, Korshunova, Mnih, and Doucet]{debortoli2024schrodinger}
Valentin De~Bortoli, Iryna Korshunova, Andriy Mnih, and Arnaud Doucet.
\newblock {{Schrödinger}} {Bridge} {Flow} for {Unpaired} {Data} {Translation}, September 2024.

\bibitem[Dempster et~al.(1977)Dempster, Laird, and Rubin]{dempster1977maximum}
A.~P. Dempster, N.~M. Laird, and D.~B. Rubin.
\newblock Maximum {{Likelihood}} from {{Incomplete Data Via}} the {{EM Algorithm}}.
\newblock \emph{Journal of the Royal Statistical Society: Series B (Methodological)}, 39\penalty0 (1):\penalty0 1--22, 1977.
\newblock ISSN 0035-9246.

\bibitem[Dowson \& Landau(1982)Dowson and Landau]{dowson1982frechet}
D.~C Dowson and B.~V Landau.
\newblock The {{Fréchet}} distance between multivariate normal distributions.
\newblock \emph{Journal of Multivariate Analysis}, 12\penalty0 (3):\penalty0 450--455, 1982.
\newblock ISSN 0047-259X.

\bibitem[Fortet(1940)]{fortet1940resolution}
Robert Fortet.
\newblock Résolution d’un systeme d’équations de {{M}}. {{Schrödinger}}.
\newblock \emph{J. Math. Pure Appl. IX}, 1:\penalty0 83--105, 1940.

\bibitem[Gushchin et~al.(2023)Gushchin, Kolesov, Mokrov, Karpikova, Spiridonov, Burnaev, and Korotin]{gushchin2023building}
Nikita Gushchin, Alexander Kolesov, Petr Mokrov, Polina Karpikova, Andrei Spiridonov, Evgeny Burnaev, and Alexander Korotin.
\newblock Building the {{Bridge}} of {{Schrödinger}}: {{A Continuous Entropic Optimal Transport Benchmark}}.
\newblock In \emph{Thirty-Seventh {{Conference}} on {{Neural Information Processing Systems Datasets}} and {{Benchmarks Track}}}, 2023.

\bibitem[Gushchin et~al.(2024)Gushchin, Kholkin, Burnaev, and Korotin]{gushchin2024lighta}
Nikita Gushchin, Sergei Kholkin, Evgeny Burnaev, and Alexander Korotin.
\newblock Light and {{Optimal Schrödinger Bridge Matching}}.
\newblock In \emph{Forty-First {{International Conference}} on {{Machine Learning}}}, 2024.

\bibitem[Ho et~al.(2020)Ho, Jain, and Abbeel]{ho2020denoising}
Jonathan Ho, Ajay Jain, and Pieter Abbeel.
\newblock Denoising {{Diffusion Probabilistic Models}}.
\newblock arXiv, 2020.

\bibitem[Jamison(1974)]{jamison1974reciprocal}
Benton Jamison.
\newblock Reciprocal processes.
\newblock \emph{Zeitschrift für Wahrscheinlichkeitstheorie und Verwandte Gebiete}, 30\penalty0 (1):\penalty0 65--86, 1974.
\newblock ISSN 1432-2064.

\bibitem[Jamison(1975)]{jamison1975markov}
Benton Jamison.
\newblock The {{Markov}} processes of {{Schrödinger}}.
\newblock \emph{Zeitschrift für Wahrscheinlichkeitstheorie und Verwandte Gebiete}, 32\penalty0 (4):\penalty0 323--331, 1975.
\newblock ISSN 1432-2064.

\bibitem[Karimi et~al.(2023)Karimi, Hsieh, and Krause]{karimi2023sinkhorn}
Mohammad~Reza Karimi, Ya-Ping Hsieh, and Andreas Krause.
\newblock Sinkhorn {{Flow}}: {{A Continuous-Time Framework}} for {{Understanding}} and {{Generalizing}} the {{Sinkhorn Algorithm}}, 2023.

\bibitem[Kloeden \& Platen(1992)Kloeden and Platen]{kloeden1992numerical}
Peter~E. Kloeden and Eckhard Platen.
\newblock \emph{Numerical {{Solution}} of {{Stochastic Differential Equations}}}.
\newblock Springer Berlin Heidelberg, 1992.

\bibitem[Korotin et~al.(2023)Korotin, Gushchin, and Burnaev]{korotin2023light}
Alexander Korotin, Nikita Gushchin, and Evgeny Burnaev.
\newblock Light {{Schrödinger Bridge}}.
\newblock In \emph{The {{Twelfth International Conference}} on {{Learning Representations}}}, 2023.

\bibitem[Liu et~al.(2023)Liu, Vahdat, Huang, Theodorou, Nie, and Anandkumar]{liu2023sb}
Guan-Horng Liu, Arash Vahdat, De-An Huang, Evangelos Theodorou, Weili Nie, and Anima Anandkumar.
\newblock I$^2${{SB}}: {{Image-to-Image Schrödinger Bridge}}.
\newblock In \emph{Proceedings of the 40th {{International Conference}} on {{Machine Learning}}}, pp.\  22042--22062. PMLR, 2023.

\bibitem[Liu et~al.(2022)Liu, Wu, Ye, and Liu]{liu2022let}
Xingchao Liu, Lemeng Wu, Mao Ye, and Qiang Liu.
\newblock Let us {{Build Bridges}}: {{Understanding}} and {{Extending Diffusion Generative Models}}.
\newblock In \emph{{{NeurIPS}} 2022 {{Workshop}} on {{Score-Based Methods}}}, 2022.

\bibitem[Léonard(2014{\natexlab{a}})]{leonard2014properties}
Christian Léonard.
\newblock Some {{Properties}} of {{Path Measures}}.
\newblock In Catherine {Donati-Martin}, Antoine Lejay, and Alain Rouault (eds.), \emph{Séminaire de {{Probabilités XLVI}}}, Lecture {{Notes}} in {{Mathematics}}, pp.\  207--230. Springer International Publishing, 2014{\natexlab{a}}.

\bibitem[Léonard(2014{\natexlab{b}})]{leonard2014survey}
Christian Léonard.
\newblock A survey of the {{Schrödinger}} problem and some of its connections with optimal transport.
\newblock \emph{Discrete \& Continuous Dynamical Systems}, 34\penalty0 (4):\penalty0 1533, 2014{\natexlab{b}}.

\bibitem[Mallasto et~al.(2022)Mallasto, Gerolin, and Minh]{mallasto2022entropyregularized}
Anton Mallasto, Augusto Gerolin, and Hà~Quang Minh.
\newblock Entropy-regularized 2-{{Wasserstein}} distance between {{Gaussian}} measures.
\newblock \emph{Information Geometry}, 5\penalty0 (1):\penalty0 289--323, 2022.
\newblock ISSN 2511-249X.

\bibitem[Marzouk et~al.(2016)Marzouk, Moselhy, Parno, and Spantini]{marzouk2016sampling}
Youssef Marzouk, Tarek Moselhy, Matthew Parno, and Alessio Spantini.
\newblock Sampling via {{Measure Transport}}: {{An Introduction}}.
\newblock In Roger Ghanem, David Higdon, and Houman Owhadi (eds.), \emph{Handbook of {{Uncertainty Quantification}}}, pp.\  1--41. Springer International Publishing, 2016.

\bibitem[Peluchetti(2021)]{peluchetti2021nondenoising}
Stefano Peluchetti.
\newblock Non-{{Denoising Forward-Time Diffusions}}.
\newblock 2021.

\bibitem[Peluchetti(2023)]{peluchetti2023diffusion}
Stefano Peluchetti.
\newblock Diffusion {{Bridge Mixture Transports}}, {{Schrödinger Bridge Problems}} and {{Generative Modeling}}.
\newblock \emph{Journal of Machine Learning Research}, 24\penalty0 (374):\penalty0 1--51, 2023.
\newblock ISSN 1533-7928.

\bibitem[Peyré \& Cuturi(2020)Peyré and Cuturi]{peyre2020computational}
Gabriel Peyré and Marco Cuturi.
\newblock Computational {{Optimal Transport}}.
\newblock 2020.

\bibitem[Ruschendorf(1995)]{ruschendorf1995convergence}
Ludger Ruschendorf.
\newblock Convergence of the {{Iterative Proportional Fitting Procedure}}.
\newblock \emph{The Annals of Statistics}, 23\penalty0 (4):\penalty0 1160--1174, 1995.
\newblock ISSN 0090-5364.

\bibitem[Shi et~al.(2023)Shi, Bortoli, Campbell, and Doucet]{shi2023diffusiona}
Yuyang Shi, Valentin~De Bortoli, Andrew Campbell, and Arnaud Doucet.
\newblock Diffusion {{Schrödinger Bridge Matching}}.
\newblock In \emph{Thirty-Seventh {{Conference}} on {{Neural Information Processing Systems}}}, 2023.

\bibitem[Somnath et~al.(2023)Somnath, Pariset, Hsieh, Martinez, Krause, and Bunne]{somnath23aligned}
Vignesh~Ram Somnath, Matteo Pariset, Ya-Ping Hsieh, Maria~Rodriguez Martinez, Andreas Krause, and Charlotte Bunne.
\newblock Aligned diffusion {S}chrödinger bridges.
\newblock In Robin~J. Evans and Ilya Shpitser (eds.), \emph{Proceedings of the Thirty-Ninth Conference on Uncertainty in Artificial Intelligence}, volume 216 of \emph{Proceedings of Machine Learning Research}, pp.\  1985--1995. PMLR, 31 Jul--04 Aug 2023.

\bibitem[Song et~al.(2021)Song, {Sohl-Dickstein}, Kingma, Kumar, Ermon, and Poole]{song2021scorebased}
Yang Song, Jascha {Sohl-Dickstein}, Diederik~P Kingma, Abhishek Kumar, Stefano Ermon, and Ben Poole.
\newblock Score-{{Based Generative Modeling}} through {{Stochastic Differential Equations}}.
\newblock In \emph{International {{Conference}} on {{Learning Representations}}}, 2021.

\bibitem[Vargas et~al.(2021)Vargas, Thodoroff, Lamacraft, and Lawrence]{vargas2021solving}
Francisco Vargas, Pierre Thodoroff, Austen Lamacraft, and Neil Lawrence.
\newblock Solving {{Schrödinger Bridges}} via {{Maximum Likelihood}}.
\newblock \emph{Entropy}, 23\penalty0 (9):\penalty0 1134, 2021.

\bibitem[Øksendal(2013)]{oksendal2013stochastic}
B.~K. Øksendal.
\newblock \emph{Stochastic {{Differential Equations}}: {{An Introduction}} with {{Applications}}}.
\newblock Universitext. Springer, 6th ed., 6th corrected printing edition, 2013.

\end{thebibliography}
\bibliographystyle{tmlr}

\clearpage
\appendix

\section{Additional Dynamics}\label{sec:additional_dynamics}

In this section we consider a simple extension to the dynamics of \cref{sec:reference_dynamics}, and refer the reader to \citet{peluchetti2021nondenoising,peluchetti2023diffusion} for the more general case.
Here, we consider the case where the reference distribution $R$ is given by the solution to:
\begin{equation}\label{eq:ref_sde_schedule}
X_0 ∼ Ψ_0,\quad dX_t = σ \sqrt{β_t}dW_t,\quad t∈[0,1],
\end{equation}
with $σ ≥ 0$, $β_t: [0,1] → ℝ_{>0}$ strictly positive and continuous.
With $b_{s:t} ≔ ∫_s^tβ_udu$, $0≤s≤t≤1$, $β_t$ is chosen such that $b_{0:1} = 1$, to disentangle the contribution of $β_t$ from the contribution of $σ$.
Indeed, under these conditions, $β_t$ defines a time-warping: if $X_t$ is the solution to \cref{eq:ref_sde}, then $X_{b_{0:t}}$ has the same distribution as the solution to \cref{eq:ref_sde_schedule}.
Consequently, the solutions to \cref{eq:sb_static_problem} and \cref{eq:eot_problem} are independent of $β_t$.

When employing \cref{eq:ref_sde_schedule}, the definitions in \cref{sec:reference_dynamics} are replaced as follows:
\begin{align}
 & R_{t|0}(□|x_0) = 𝒩(x_0,σ^2b_{0:t}),                                                                                          \\
 & R_{t|0,1}(□|x_0,x_1) = 𝒩(x_0b_{t:1} + x_1b_{0:t}, σ^2b_{0:t}b_{t:1}),                                                        \\
 & μ_{01}(x_t,t,x_1) ≔ σ^2 β_t ∇_{x_t} \log r_{1|t}(x_1|x_t) = \frac{β_t}{b_{t:1}}(x_1 - x_t),                                  \\
 & υ_{01}(x_t,t,x_0) ≔ σ^2β_t∇_{x_t}\log r_{t|0}(x_t|x_0) = \frac{β_t}{b_{0:t}}(x_0 - x_t),                                     \\
 & γ_{01}(x_t,t,x_0,x_1) ≔ σ^2β_t∇_{x_t}\log r_{t|01}(x_t|x_0,x_1) = \frac{β_t}{b_{0:t}b_{t:1}}(x_0b_{t:1} + x_1b_{0:t} - x_t).
\end{align}

\section{Proofs}\label{sec:proofs}

\bmtwofull*
\begin{proof}
Define $Q$ associated with
\begin{equation}\label{eq:ref_sde_bwd}
X_1 ∼ Ψ_1,\quad dX_t = σ dW_t,\quad t∈[1,0],\tag{$\overleftarrow{Q}$}
\end{equation}
which is not the time reversal of \cref{eq:ref_sde}, but $R_{{}|0,1} = Q_{{}|0,1}$.

Firstly, consider the case of initial null drifts: $μ_f^{(0)} = υ_b^{(0)} = 0$, corresponding to $F^{(0)} = Ψ_0R_{{}|0} = Ψ_0R_{1|0}R_{{}|0,1} = F^{(0)}_{0,1}R_{{}|0,1} ∈ 𝒮$ and $B^{(0)} = Ψ_1Q_{{}|1} = Ψ_1Q_{0|1}R_{{}|0,1} = B^{(0)}_{0,1}R_{{}|0,1} ∈ 𝒮$.
As $B^{(0)} = Π^{B^{(0)}_{0,1}} = M^{B^{(0)}_{0,1}}$, we have $F^{(1)} = Ψ_0M^{B^{(0)}_{0,1}}_{{}|0} = Ψ_0B^{(0)}_{{}|0} = Ψ_0B^{(0)}_{1|0}R_{{}|0,1} ∈ 𝒮$.
As $F^{(0)} = Π^{F^{(0)}_{0,1}} = M^{F^{(0)}_{0,1}}$, $B^{(1)} = Ψ_1M^{F^{(0)}_{0,1}}_{{}|1} = Ψ_1F^{(0)}_{{}|1} = Ψ_1F^{(0)}_{0|1}R_{{}|0,1} ∈ 𝒮$.
By induction, $F^{(i)} = Ψ_0B^{(i-1)}_{1|0}R_{{}|0,1} ∈ 𝒮$ and $B^{(i)} = Ψ_1F^{(i-1)}_{0|1}R_{{}|0,1} ∈ 𝒮$, $i ≥ 1$.
We now construct two forward-backward sequences.
For the sequence $F^{(0)}, B^{(1)}, F^{(2)}, …$, we have $F^{(0)}_{0,1} = Ψ_0R_{1|0}$, $B^{(1)}_{0,1} = Ψ_1F^{(0)}_{0|1}$, $F^{(2)}_{0,1} = Ψ_0B^{(1)}_{1|0}$, … which are the static IPF iterates: one marginal gets replaced at a time keeping the conditional distribution fixed.
In the same way, for $B^{(0)}, F^{(1)}, B^{(2)}, …$, we have $B^{(0)}_{0,1} = Ψ_1Q_{0|1}$, $F^{(1)}_{0,1} = Ψ_0B^{(0)}_{1|0}$, $B^{(2)}_{0,1} = Ψ_1F^{(1)}_{0|1}$, … which are again the static IPF iterates (for the backward formulation of the dynamic SB problem, i.e.\ via $\overleftarrow{Q}_{{}|0}$ as reference measure instead of $R_{{}|0}$, and switched marginal distributions).
As each pair $F^{(i)}$, $B^{(i)}$ is of the form $F^{(i)} = F^{(i)}_{0,1}R_{{}|0,1}$, $B^{(i)} = B^{(i)}_{0,1}R_{{}|0,1}$, we also recover the dynamic DIPF iterates.

Secondly, consider $μ_f^{(0)}$ and $υ_b^{(0)}$ both corresponding to the BM transport based on the given coupling: $I^{(0)}=M^{C_{0,1}}$, $F^{(0)} = B^{(0)} = I^{(0)}$.
Then, looking separately at either of the sequences $F^{(i)}$, $i ≥ 1$, and $B^{(i)}$, $i ≥ 1$, we obtain that $F^{(i)} = B^{(i)} = I^{(i)}$, $i ≥ 1$.
\end{proof}

\lossint*
\begin{proof}
We consider only $𝕃_f(μ_f;υ_b)$, the arguments for $𝕃_b(υ_b;μ_f)$ are symmetric.
By Girsanov Theorem \citep{oksendal2013stochastic} and by the marginal-conditional decomposition of Kullback-Leibler divergences we have
\begin{align*}
 & 𝕂𝕃(M^{B_{0,1}} \TO F) = 𝕂𝕃(B_0 \TO Ψ_0) + 𝔼_{Π^{B_{0,1}}}\Big[\frac{1}{2}∫_0^1‖μ_f(X_t,t) - μ_m^{B_{0,1}}(X_t,t)‖^2dt\Big],     \\
 & 𝕂𝕃(Π^{B_{0,1}} \TO F) = 𝕂𝕃(B_0 \TO Ψ_0) + 𝔼_{Π^{B_{0,1}}}\Big[\frac{1}{2}∫_0^1‖μ_f(X_t,t) - μ_π^{B_{0,1}}(X_t,t,X_0)‖^2dt\Big], \\
 & 𝕂𝕃(Π^{B_{0,1}} \TO ℱ) = 𝕂𝕃(B_0 \TO Ψ_0) + 𝔼_{Π^{B_{0,1}}}\Big[\frac{1}{2}∫_0^1‖μ_f(X_t,t) - μ_{01}(X_t,t,X_1)‖^2dt\Big]         \\
 & \quad = 𝕂𝕃(B_0 \TO Ψ_0) + 𝕃_f(μ_f;υ_b),
\end{align*}
where $μ_π^{B_{0,1}}(X_t,t,X_0) ≔ 𝔼_{Π^{B_{0,1}}}[μ_{01}(X_t,t,X_1)|X_t,X_0]$, $μ_m^{B_{0,1}}(X_t,t) ≔ 𝔼_{Π^{B_{0,1}}}[μ_{01}(X_t,t,X_1)|X_t]$, and $ℱ$ is distribution of the non-Markov diffusion solution to the auxiliary SDE
\begin{equation}\label{eq:bm2_ff}
X_0 ∼ Ψ_0,\quad dX_t = [μ_f(X_t,t) - μ_{01}(X_t,t,X_1) + μ_π^{B_{0,1}}(X_t,t,X_0)]dt + σdW_t,\quad t∈[0,1].\tag{$ℱ$}
\end{equation}
By the tower property of conditional expectations and by the conditional Jensen inequality it follows that
\begin{equation*}\begin{aligned}
 & 𝕂𝕃(Π^{B_{0,1}} \TO ℱ) - 𝕂𝕃(Π^{B_{0,1}} \TO F)                                                                                 \\
 & = 𝔼_{Π^{B_{0,1}}}\Big[\frac{1}{2}∫_0^1 ‖μ_f(X_t,t) - μ_{01}(X_t,t,X_1)‖^2 - ‖μ_f(X_t,t) - μ_π^{B_{0,1}}(X_t,t,X_0)‖^2 dt\Big] \\
 & = 𝔼_{Π^{B_{0,1}}}\Big[\frac{1}{2}∫_0^1 ‖μ_{01}(X_t,t,X_1)‖^2 - ‖μ_π^{B_{0,1}}(X_t,t,X_0)‖^2 dt\Big] = C_1(B) ≥ 0.
\end{aligned}\end{equation*}
By the Pythagorean property of the BM transport \citep{liu2022let,peluchetti2023diffusion}
\begin{equation*}
𝕂𝕃(Π^{B_{0,1}} \TO F) - 𝕂𝕃(M^{B_{0,1}} \TO F) = 𝕂𝕃(Π^{B_{0,1}} \TO M^{B_{0,1}}) = K(B) ≥ 0.
\end{equation*}
Taking $C_2(B) = C_1(B) + K(B)$ completes the proof.
\end{proof}

\stabler*
\begin{proof}
By the marginal-conditional decomposition of Kullback-Leibler divergences
\begin{equation*}\begin{aligned}
 & 𝕂𝕃(F \TO B') = 𝕂𝕃(F_{0,1} \TO B'_{0,1}) + 𝔼_{F_{0,1}}[𝕂𝕃(F_{{}|0,1} \TO B'_{{}|0,1})], \\
 & 𝕂𝕃(F \TO F') = 𝕂𝕃(F_{0,1} \TO F'_{0,1}) + 𝔼_{F_{0,1}}[𝕂𝕃(F_{{}|0,1} \TO F'_{{}|0,1})],
\end{aligned}\end{equation*}
and $B'_{{}|0,1} = F'_{{}|0,1} = R_{{}|0,1}$, hence
\begin{equation*}\begin{aligned}
 & F^{(λ)} ≔ \argmin_{F ∈ 𝒫(Ψ_0,⋅)} λ𝕂𝕃(F_{0,1} \TO B'_{0,1}) + (1 - λ)𝕂𝕃(F_{0,1} \TO F'_{0,1}) + 𝔼_{F_{0,1}}[𝕂𝕃(F_{{}|0,1} \TO R_{{}|0,1})], \\
 & B^{(λ)} ≔ \argmin_{B ∈ 𝒫(⋅,Ψ_1)} λ𝕂𝕃(B_{0,1} \TO F'_{0,1}) + (1 - λ)𝕂𝕃(B_{0,1} \TO B'_{0,1}) + 𝔼_{B_{0,1}}[𝕂𝕃(B_{{}|0,1} \TO R_{{}|0,1})], \\
\end{aligned}\end{equation*}
and thus $F^{(λ)}_{{}|0,1} = B^{(λ)}_{{}|0,1} = R_{{}|0,1}$, which completes the proof.
\end{proof}

\stables*
\begin{proof}
In view of \cref{thm:bm2_partial_r}, we have to verify that $F^{(λ)}_{0,1}, B^{(λ)}_{0,1}$ solve the EOT problems \cref{eq:eot_problem} for some marginal distributions if $F'_{0,1}, B'_{0,1}$ do.
For simplicity, we assume that all of $F^{(λ)}_{0,1}, B^{(λ)}_{0,1}, F'_{0,1}, B'_{0,1}$ admits positive densities on $ℝ^{d{×}d}$, and that $Ψ_0$ and $Ψ_1$ admits positive densities on $ℝ^d$.
The steps of this proof carry over to the more general measure-theoretic setting.

We know that $f^{(λ)}_{0,1}(x_0,x_1) = ψ_0(x_0)f^{(λ)}_{1|0}(x_1|x_0)$ and $b^{(λ)}_{0,1}(x_0,x_1) = b^{(λ)}_{0|1}(x_0|x_1)ψ_1(x_1)$, where $f^{(λ)}_{1|0}(x_1|x_0) ∝ b'_{1|0}(x_1|x_0)^λ f'_{1|0}(x_1|x_0)^{1-λ}$ and $b^{(λ)}_{0|1}(x_0|x_1) ∝ f'_{0|1}(x_0|x_1)^λ b'_{0|1}(x_0|x_1)^{1-λ}$ (see \cref{sec:bm2_theory}).
On the other hand
\begin{align*}
 & f'_{0,1}(x_0,x_1) = \exp\Big\{ϕ_0^{f'}(x_0) + ϕ_1^{f'}(x_1) - \frac{κ(x_0,x_1)}{ε}\Big\}, \\
 & b'_{0,1}(x_0,x_1) = \exp\Big\{ϕ_0^{b'}(x_0) + ϕ_1^{b'}(x_1) - \frac{κ(x_0,x_1)}{ε}\Big\},
\end{align*}
for the Schrödinger potentials\footnote{We formulate the potential with respect to the Lebesgue measure on $ℝ^d$.} $ϕ_0^{f'}(x_0), ϕ_1^{f'}(x_1)$ and $ϕ_0^{b'}(x_0) + ϕ_1^{b'}(x_1)$ \citep{leonard2014properties}.
It follows by direct computation that $f^{(λ)}_{0,1}(x_0,x_1)$ and $b^{(λ)}_{0,1}(x_0,x_1)$ satisfy:
\begin{align*}
 & f^{(λ)}_{0,1}(x_0,x_1) = \exp\Big\{ϕ_0^{f,λ}(x_0) + ϕ_1^{f,λ}(x_1) - \frac{κ(x_0,x_1)}{ε}\Big\}, \\
 & b^{(λ)}_{0,1}(x_0,x_1) = \exp\Big\{ϕ_0^{b,λ}(x_0) + ϕ_1^{f,λ}(x_1) - \frac{κ(x_0,x_1)}{ε}\Big\},
\end{align*}
for some other Schrödinger potentials $ϕ_0^{f,λ}(x_0), ϕ_1^{f,λ}(x_1)$ and $ϕ_0^{b,λ}(x_0), ϕ_1^{f,λ}(x_1)$.
\end{proof}

\section{Additional Results}\label{sec:additional_results}

\subsection{Infinitesimal Minimization, Gaussian Case}\label{sec:bm2_inf_gauss}

\begin{figure}
\centering
\includegraphics{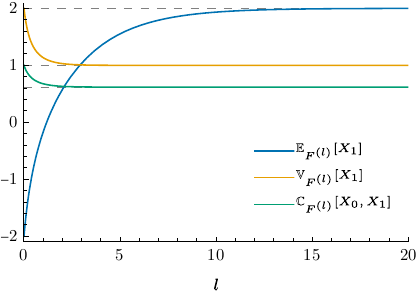}
\caption{Algorithmic-time $l$ evolution of $𝔼_{F^{(l)}}[X_1]$, $𝕍_{F^{(l)}}[X_1]$, $ℂ_{F^{(l)}}[X_0,X_1]$, compared with $𝔼_S[X_1]$, $𝕍_S[X_1]$, $ℂ_S[X_0,X_1]$ as dashed gray lines.}\label{fig:bm2_infmin_gauss_f}
\end{figure}

Consider the one-dimensional case $d=1$, with target Gaussian marginal distributions $Ψ_0 = 𝒩(μ_0,σ^2_0)$ and $Ψ_1 = 𝒩(μ_1,σ^2_1)$, and a reference diffusion distribution $R$ associated with \cref{eq:ref_sde}.
In this setting, the solution to the static Schrödinger bridge problem \cref{eq:sb_static_problem} is known analytically and is given by a bivariate Gaussian distribution \citep{mallasto2022entropyregularized}.

We hypothesize that conditional Gaussian densities for $f^{(l)}_{1|0}(x_1|x_0)$ and $b^{(l)}_{0|1}(x_0|x_1)$ solve \cref{eq:bm2_partial_static_dyn_2}.
Specifically, we propose $F^{(l)}_{1|0} = 𝒩(A^f_l x_0 + a^f_l, v^f_l)$ and $B^{(l)}_{0|1} = 𝒩(A^b_l x_1 + a^b_l, v^b_l)$, where $A^f_l,a^f_l,A^b_l, a^b_l ∈ ℝ$ and $v^f_l,v^b_l∈ℝ_{>0}$ are algorithmic-time dependent scalar parameters.
By construction, $F^{(l)}_{0} = 𝒩(μ_0,σ^2_0)$ and $B^{(l)}_{1} = 𝒩(μ_1,σ^2_1)$ for each $l ≥ 0$.
Substituting these expressions for $f^{(l)}_{1|0}(x_1|x_0)$ and $b^{(l)}_{0|1}(x_0|x_1)$ into \cref{eq:bm2_partial_static_dyn_2} yields a six-dimensional ODE system in the parameters.
The initial conditions are $A^b_0 = A^f_0 = 1$, $a^f_0 = a^b_0 = 0$, $v^f_0 = v^b_0 = σ^2$, corresponding to initial null drift terms for \cref{eq:bm2_f,eq:bm2_b}, as discussed in \cref{sec:bm2_theory}.

To numerically solve the ODE and determine the values of $A^f_l, a^f_l, v^f_l, A^b_l, a^b_l, v^b_l$ over $l ∈ [0,L]$ for some $L > 0$, we evaluate \cref{eq:bm2_partial_static_dyn_2} for three different pairs of $(x_0,x_1)$.
This provides sufficient constraints to identify the parameters.
Subsequently, we verify that the proposed functional forms for $f^{(l)}_{1|0}(x_1|x_0)$ and $b^{(l)}_{0|1}(x_0|x_1)$ indeed solve \cref{eq:bm2_partial_static_dyn_2}.

We examine the scenario where $μ_0=-2, μ_1=2, σ_0=σ_1=σ=1$.
\cref{fig:bm2_infmin_gauss_f} illustrates the evolution of $𝔼_{F^{(l)}}[X_1]$, $𝕍_{F^{(l)}}[X_1]$, and $ℂ_{F^{(l)}}[X_0,X_1]$ over algorithmic time $l$.
These quantities represent the mean and variance of $X_1$ and the covariance between $X_0$ and $X_1$ according to $F^{(l)}$, respectively.
The corresponding values $𝔼_S[X_1]$, $𝕍_S[X_1]$, and $ℂ_S[X_0,X_1]$ for the static Schrödinger bridge solution $S_{0,1}$ from \citet{mallasto2022entropyregularized} are depicted as dashed gray lines, demonstrating convergence.

\subsection{Results for Additional Discretization Intervals}\label{sec:cbw_discretizations}

In \cref{tab:cbw_discretization} we report the results for metric \cref{eq:metric_cbw} obtained by considering the BM$^2$, I-BM and DIPF methodologies for different discretization intervals $Δt = 1/T$ where $T$ is the number of time-steps.
For all methods we employ the same number of time-steps at training and inference (testing) time.
We recall that in \cref{sec:numerical} $200$ time-steps have been employed to produce the results of \cref{tab:kl_results,tab:cbw_results}, and that we rely exclusively on the Euler–Maruyama discretization scheme.

\begin{table}[h!]
\ra{1.2}
\centering
\small
\begin{tabular}{@{}l*{12}r@{}}\toprule
            & \multicolumn{4}{c}{$ε \feq 0.1$} & \multicolumn{4}{c}{$ε \feq 1$} & \multicolumn{4}{c}{$ε \feq 10$}                                                                                                                                                                                     \\
\cmidrule(lr){2-5} \cmidrule(lr){6-9} \cmidrule(l){10-13}
Method      & $d \feq 2$                       & $d \feq 16$                    & $d \feq 64$                     & $d \feq 128$      & $d \feq 2$       & $d \feq 16$      & $d \feq 64$       & $d \feq 128$      & $d \feq 2$       & $d \feq 16$       & $d \feq 64$        & $d \feq 128$        \\
\midrule
SB(100)     & \fms{0.08}{0.00}                 & \fms{0.10}{0.00}               & \fms{0.37}{0.00}                & \fms{0.94}{0.00}  & \fms{0.10}{0.00} & \fms{0.19}{0.00} & \fms{0.49}{0.00}  & \fms{1.25}{0.00}  & \fms{0.12}{0.00} & \fms{2.96}{0.00}  & \fms{14.86}{0.02}  & \fms{28.90}{0.03}   \\
SB(50)      & \fms{0.23}{0.00}                 & \fms{0.22}{0.00}               & \fms{0.44}{0.00}                & \fms{1.00}{0.00}  & \fms{0.15}{0.00} & \fms{0.24}{0.00} & \fms{0.61}{0.00}  & \fms{1.49}{0.00}  & \fms{0.13}{0.00} & \fms{7.24}{0.01}  & \fms{37.24}{0.04}  & \fms{72.34}{0.06}   \\
\midrule[0.25pt]
BM$^2$(100) & \fms{1.17}{0.79}                 & \fms{5.08}{0.22}               & \fms{7.12}{0.71}                & \fms{7.94}{0.78}  & \fms{0.14}{0.04} & \fms{0.41}{0.03} & \fms{1.71}{0.14}  & \fms{8.91}{0.35}  & \fms{0.15}{0.02} & \fms{4.06}{0.03}  & \fms{55.10}{1.13}  & \fms{295.82}{10.60} \\
BM$^2$(50)  & \fms{0.80}{0.52}                 & \fms{4.59}{0.41}               & \fms{6.46}{0.87}                & \fms{8.47}{0.73}  & \fms{0.21}{0.09} & \fms{0.48}{0.03} & \fms{1.85}{0.10}  & \fms{9.57}{0.27}  & \fms{0.16}{0.01} & \fms{8.17}{0.04}  & \fms{78.76}{1.85}  & \fms{202.22}{41.38} \\
\midrule[0.25pt]
DIPF(100)   & \fms{7.28}{4.19}                 & \fms{15.72}{1.31}              & \fms{16.44}{0.89}               & \fms{26.69}{0.71} & \fms{1.43}{0.45} & \fms{5.96}{0.32} & \fms{11.98}{0.82} & \fms{22.60}{0.86} & \fms{1.14}{0.07} & \fms{7.36}{0.16}  & \fms{85.42}{2.64}  & \fms{212.19}{1.19}  \\
DIPF(50)    & \fms{6.90}{3.01}                 & \fms{12.37}{1.40}              & \fms{14.98}{2.16}               & \fms{26.91}{2.71} & \fms{1.44}{0.33} & \fms{5.09}{0.43} & \fms{10.73}{0.29} & \fms{20.20}{0.93} & \fms{2.18}{0.06} & \fms{11.32}{0.18} & \fms{75.80}{1.02}  & \fms{226.81}{0.92}  \\
\midrule[0.25pt]
I-BM(100)   & \fms{0.98}{0.22}                 & \fms{4.54}{0.98}               & \fms{6.35}{0.84}                & \fms{13.60}{1.23} & \fms{0.18}{0.04} & \fms{0.51}{0.03} & \fms{2.09}{0.23}  & \fms{7.30}{1.06}  & \fms{0.15}{0.02} & \fms{5.95}{0.10}  & \fms{113.91}{0.93} & \fms{494.32}{3.99}  \\
I-BM(50)    & \fms{1.58}{0.61}                 & \fms{4.21}{0.20}               & \fms{6.35}{1.01}                & \fms{13.75}{2.93} & \fms{0.28}{0.05} & \fms{0.60}{0.04} & \fms{1.96}{0.25}  & \fms{7.31}{1.24}  & \fms{0.16}{0.02} & \fms{9.14}{0.07}  & \fms{114.40}{0.89} & \fms{432.20}{5.82}  \\
\bottomrule
\end{tabular}
\caption{Results analogous to \cref{tab:cbw_results} but for varying discretization intervals $Δt = 1/T$, where the number of time-steps $T$ is indicated in parentheses after each method name.}
\label{tab:cbw_discretization}
\end{table}

\clearpage

\section{Python Code}\label{sec:code}

\begin{listing}[H]
\begin{minted}[linenos,numbersep=5pt,fontsize=\footnotesize,texcomments,frame=lines,framesep=2mm]{python}
# dimensions: B: batch; D: data; T: time\_steps + 1
# required: sample\_0(batch\_dim, device), sample\_1(batch\_dim, device), fwd\_drift\_fn(x, t), bwd\_drift\_fn(x, t)
import torch as th

# sampling from $R_{t|0,1}$ \cref{eq:ref_sde_tdd}: (B, D), (B, D), (B,), () -> (B, D)
def sample_bridge(x_0, x_1, t, sigma):
    B, D = x_0.shape
    mean_t = (1 - t[..., None]) * x_0 + t[..., None] * x_1  # (B, D)
    var_t = sigma**2 * t[..., None] * (1 - t[..., None])    # (B, D)
    z_t = th.randn_like(x_0)                                # (B, D)
    x_t = mean_t + th.sqrt(var_t) * z_t                     # (B, D)
    return x_t

# fwd BM target \cref{eq:μ01}: (B, D), (B, D), (B,) -> (B, D)
def fwd_target(x_t, x_1, t):
    return (x_1 - x_t) / (1 - t[..., None])  # (B, D)

# fwd BM target \cref{eq:υ01}: (B, D), (B, D), (B,) -> (B, D)
def bwd_target(x_t, x_0, t):
    return (x_0 - x_t) / t[..., None]  # (B, D)

# Euler–Maruyama discretization scheme: fn(x, t), (B, D), (T), () -> (B, D)
def discretization(drift_fn, initial_value, times, sigma):
    B, D = initial_value.shape
    times = times[..., None].expand(-1, B)                                   # (T, B)
    x_prev_t = initial_value                                                 # (B, D)
    for prev_t, t in zip(times[:-1], times[1:]):                             # (B), (B)
        dt = t - prev_t                                                      # (B)
        drift_t = drift_fn(x_prev_t, prev_t)                                 # (B, D)
        drift_part_t = drift_t * dt[..., None]                               # (B, D)
        eps_t = th.randn_like(x_prev_t)                                      # (B, D)
        diffusion_part_t = (sigma * th.sqrt(th.abs(dt)))[..., None] * eps_t  # (B, D)
        x_t = x_prev_t + drift_part_t + diffusion_part_t                     # (B, D)
        x_prev_t = x_t                                                       # (B, D)
    return x_t

# BM$^2$ loss computation: fn(b, d), fn(b, d), fn(x, t), fn(x, t), (), (), (), () -> ()
def sample_loss(sample_0, sample_1, fwd_drift_fn, bwd_drift_fn, batch_dim, time_steps, sigma, device):
    # sample from the target marginals:
    f_0 = sample_0(batch_dim, device)                                   # (B, D)
    b_1 = sample_1(batch_dim, device)                                   # (B, D)
    # sample according to current \cref{eq:bm2_f} and \cref{eq:bm2_b}:
    fwd_times = th.linspace(0.0, 1.0, time_steps + 1, device=device)    # [0, 1/time\_steps, ..., 1]
    bwd_times = th.linspace(1.0, 0.0, time_steps + 1, device=device)    # [1, ..., 1/time\_steps, 0]
    f_1 = discretization(fwd_drift_fn, f_0, fwd_times, sigma).detach()  # (B, D)
    b_0 = discretization(bwd_drift_fn, b_1, bwd_times, sigma).detach()  # (B, D)
    # sample time and mixture processes based on $F_{0,1}(θ)$ and $B_{0,1}(θ)$:
    t = th.rand((batch_dim,), device=device)                            # (B)
    pi_f_t = sample_bridge(f_0, f_1, t, sigma)                          # (B, D)
    pi_b_t = sample_bridge(b_0, b_1, t, sigma)                          # (B, D)
    # define regression targets and model predictions:
    target_f_t = fwd_target(pi_b_t, b_1, t)                             # (B, D)
    target_b_t = bwd_target(pi_f_t, f_0, t)                             # (B, D)
    prediction_f_t = fwd_drift_fn(pi_b_t, t)                            # (B, D)
    prediction_b_t = bwd_drift_fn(pi_f_t, t)                            # (B, D)
    # compute loss:
    loss_f_t = th.sum((target_f_t - prediction_f_t)**2, dim=1) / 2      # (B)
    loss_b_t = th.sum((target_b_t - prediction_b_t)**2, dim=1) / 2      # (B)
    loss_t = th.mean(loss_f_t + loss_b_t)                               # ()
    return loss_t
\end{minted}
\caption{Basic implementation of BM$^2$ loss computation (\cref{alg:bm2_obj}) in PyTorch.}\label{code:bm2}
\end{listing}

\end{document}